\documentclass[twocolumn,10pt]{asme2ej}

\usepackage{epsfig} 
\usepackage{amsfonts}
\usepackage{graphicx}
\usepackage{amssymb}
\usepackage{epstopdf}
\usepackage{amsmath}
\usepackage{multirow}
\usepackage{color}
\usepackage{caption}
\usepackage{subfigure}
\usepackage[sort]{cite}
\usepackage{algorithmic}
\usepackage{algorithm}
\usepackage{soul}

\usepackage{color}
\usepackage{amsmath}
\usepackage{amssymb}
\usepackage{graphicx}
\usepackage{comment,xspace}
\usepackage{fancybox}

\usepackage{todonotes}

\newcommand{\real}{{\mathbb{R}}}
\newcommand{\reals}{\real}

\renewcommand{\natural}{{\mathbb{N}}}
\newcommand{\naturals}{\natural}

\newtheorem{theorem}{Theorem}[section]

\newtheorem{lemma}[theorem]{Lemma}

\newtheorem{assumption}[theorem]{Assumption}

\title{Trading Safety Versus Performance: Rapid Deployment of Robotic Swarms \\ with Robust Performance Constraints} 

{\small
\author{Yin-Lam Chow
    \affiliation{Institute for Computational \& Mathematical Engineering\\
	Stanford University\\ Stanford, CA 94305\\
        Email: ychow@stanford.edu
    }
}

\author{Marco Pavone
    \affiliation{Aeronautics
and Astronautics \\
	Stanford University\\ Stanford, CA 94305\\
        Email: pavone@stanford.edu
    }
}

\author{Brian M. Sadler 
    \affiliation{ Army Research Lab\\
	Adelphi, MD 20783\\
        Email: brian.m.sadler6.civ@mail.mil
    }
}
\author{Stefano Carpin
    \affiliation{
School of Engineering\\
	University of California\\
	Merced, CA 95343\\
    Email: scarpin@ucmerced.edu
    }	
}
}

\begin{document}

\maketitle    

\begin{abstract}
{\em In this paper we consider a stochastic deployment problem, where a robotic swarm is tasked with the objective of positioning at least one robot at each of a set of pre-assigned targets  while meeting a temporal deadline. Travel times and failure rates are stochastic but related, inasmuch as failure rates increase with speed. To maximize chances of success while meeting the deadline, a control strategy has therefore to balance safety and performance. Our approach is to cast the problem within the theory of constrained Markov Decision Processes, whereby we seek to compute policies that maximize the probability of successful deployment while ensuring that the expected duration of the task is bounded by a given deadline. To account for uncertainties in the problem parameters, we consider a robust formulation and we propose efficient solution algorithms, which are of independent interest. Numerical experiments  confirming our theoretical results are presented and discussed.}
\end{abstract}


\section{Introduction}

Recent technological advances have made possible the deployment of robotic swarms comprising hundreds of small, minimalistic, and inexpensive ground, air, and underwater mobile platforms \cite{Bonabeau.ea:99}. The potential advantages of robotic swarms are numerous. For example, it is possible to reduce the total implementation and operation cost, and add flexibility, robustness, and modularity with respect to monolithic approaches. Accordingly, planning and control for multi-robot systems (including robotic swarms) has received tremendous attention in the past decade from the control, robotics, and artificial intelligence communities, see  \cite{BulloRN} and references therein.

When deploying robotic swarms, safety and speed often come as \emph{contradicting} objectives: the faster it is required for robots to accomplish a task, the higher the chance that the task is not accomplished.
For example, robot batteries discharge more rapidly when robots operate at high velocities, thus shortening operational
time and increasing the probability that the assigned task will not be completed. Similarly, sensor accuracy often 
decreases with speed, thereby interfering with the ability of a robot to accomplish its mission (see, e.g., \cite{Pavlic2009}). Accordingly, the goal of this paper is to devise analysis tools and control algorithms to address such safety/performance trade-off within the context of the \emph{stochastic deployment problem}, whereby it is desired that a robotic swarm deploys within a given map so that each of a set of target locations is reached by at least one robot.  The set up is stochastic in the sense that traversal times and robot failures represent stochastic events,
 and we propose an approach that is robust to additive errors in the modeling of  traversal times. 

Deployment problems have been recently studied along a number of dimensions and with a variety of techniques, including coverage control through locational optimization \cite{cortes2004coverage, Schwager.ea:RSS06}, robot dispersion through minimalistic control strategies (e.g., random walks)\cite{Morlok2004,scout}, workload sharing in dynamic environments through distributed optimization \cite{pavone2011distributed}, deployment with  temporal logic specifications, e.g., through formal methods tools \cite{CarpinICRA2013, Kloetzer.Belta:TR07, ding2011automatic}, deployment under communication constraints through Partially Observable Markov Decision Processes \cite{BatalinTraRO2007,Fink2013, matignon2012coordinated}, and strategic deployment with complex mission specifications through the theory of constrained Markov Decision Processes \cite{ding2013strategic}. This list is necessarily incomplete (see \cite{BulloRN} for a thorough literature review on this problem), but it however outlines the main facets and tools employed for this problem in the last decade, and shows the research shift from results mostly focusing on the quality of the steady-state solution (e.g., coverage control) to studies enforcing constraints on the transient (e.g., strategic deployment).

Yet, despite the wealth of results available for the deployment problem and, more in general, for multi-robot coordination, to the best of the authors' knowledge no results exist today  that explore the trade-off between safe and rapid deployment in stochastic environments (perhaps
with the exception of \cite{Klavis2011}, where however a slightly different problem is considered). The purpose of this paper is to bridge this gap. Our strategy is to frame the problem within the paradigm of constrained Markov Decision Processes (CMDPs). Specifically, our contributions are as follows:
\begin{enumerate}
\item We show how the rapid \emph{single-robot}  deployment problem, where it is desired to trade-off safety versus performance,  can be modeled as a  \emph{robust} CMDP (RCMDP). The robustness aspect of the formulation stems from the fact that the parameters of the constraint cost functions belong to a (known) uncertainty set.
\item We show that an optimal policy for a RCMDP can be determined using a
linear optimization problem based on the concept of occupation measures developed
for CMDPs. 
\item We demonstrate how the linear optimization problem can be efficiently solved
using a transformation that reduces the number of constraints from (possibly) exponential 
to linear in the size of the problem.
\item We illustrate how RCMDPs can be used to determine both centralized and (minimalistic) decentralized policies to solve the
rapid \emph{multi-robot} deployment problem.
\end{enumerate}

The rest of the paper is organized as follows. Section~\ref{sec:CMDP} 
formally introduces
the CMDP model and provides well-known results, together with pointers to selected references.
The RCMDP model is introduced in Section \ref{sec:CMDP} as well. 
Section \ref{sec:formulation} defines the rapid deployment
problem  for both the single- and multi-robot cases, and  shows how RCMDPs together with a task assignment strategy can be used for its solution.
An efficient algorithm for the RCDMP problem is presented in Section \ref{sec:dp}, while 
in Section \ref{sec:Efficient_TA} we describe the task assignment algorithm used in our approach. Section \ref{sec:swarm_algo} presents the overall solving algorithm.
 Detailed numerical experiments are presented  in Section \ref{sec:simulation}, and
in Section \ref{sec:conclusions} we draw our conclusions and provide directions for future research.

\section{Background Material }
\label{sec:CMDP}
In this section we provide a brief
summary about Markov Decision Processes and  constrained Markov Decision Processes.
We limit our discussion to finite MDPs and we
embrace the notation used in \cite{Altman1996}. For a complete treatment on this subject, the reader is referred to \cite{BertsekasDPVol1,PutermanMDP} and \cite{altman1999constrained} for MDPs and CMDPs, respectively.

\subsection{Total Cost Markov Decision Processes}
A \emph{stationary}, finite MDP is  a quadruple $\mathbf{X},A,c,\mathcal{P}$ where:

\begin{enumerate}
\item $\mathbf{X}$ is a finite set of $n = |\mathbf{X}|$ states. The temporal evolution 
of the state of an MDP is stochastic and the state at time $t$ is 
given by the random variable $X_t$.
\item $A$ is a collection of $n$ finite sets. Each set is
indicated as $A(x)$ and represents the set of actions that can
be applied when the system is in state $x$. It is convenient
to define the set $\mathcal{K} = \{(x,a): x\in \mathbf{X}, a\in A(x)\}$.
$\mathcal{K}$ is the set of allowable state/action pairs.
In general, the action taken at time $t$ is a random variable indicated
by the letter $A_t$.
\item $c: \mathcal{K}\rightarrow \mathbb{R}_{\geq 0}$ is the function defining 
the objective cost incurred when applying action $a \in A(x)$
while in state $x\in \mathbf{X}$. We assume that these costs are non-negative.
\item $\mathcal{P}_{xy}^a$ is the one step transition probability from state $x$
to state $y$ when action $a$ is applied, i.e., $\mathcal{P}_{xy}^a=\Pr[{X_{t+1}=y}|X_t=x,A_t=a]$.
\end{enumerate}
The above model is {\em stationary} because costs and transition
probabilities do not depend on time. Based on the above definitions, the sequence of states and actions
over time constitutes a stochastic process that we denote as $(X_t,A_t)$. Without loss of generality (since the model is stationary), we  assume that the evolution starts at $t=0$.

The optimal control of an MDP entails the determination of a closed-loop 
policy $\pi$ defining which action should be applied   in order to minimize
an aggregate (sum) objective function of the
 objective costs. A policy $\pi$ induces a mass distribution\footnote{Such mass distribution not only exists, but
can be explicitly computed.} over the realizations
of the stochastic process $(X_t,A_t)$.
Let $\Pi_{\text{D}}$ be the set of closed-loop, Markovian, stationary, and \emph{deterministic} policies $\pi: \mathbf{X}\rightarrow A$.
It is well known that for MDPs there is no loss of optimality in restricting the attention to policies in $\Pi_{\text{D}}$ (instead, e.g., of also considering history-dependent or randomized policies). On the other hand, as we will discuss later,
randomization is needed for optimal policies for CMDPs.

Depending on the specific form of the objective function, MDP problems can be categorized into four main classes: finite horizon MDPs, infinite horizon MDPs with discounted cost, average-cost infinite horizon MDPs, and total cost MDPs (or optimal stopping MDPs), where the cost is on an infinite horizon but the state will eventually enter an absorbing set where no additional  costs are incurred. Given their relevance to deployment problems, in this paper we focus on total cost MDPs.
In the literature one finds  three types of total cost MDPs: (i) the 
transient MDPs, for which the total expected time spent in each state is finite under any policy, (ii) the absorbing MDPs, for which the total expected ``life time" 
of the system is finite under any policy, and (iii) contracting MDPs \cite{Altman1996}. One can show that all three types  of total cost MDPs are equivalent under the assumption of a finite state space \cite{Altman1996}. Accordingly, in this paper we focus on \emph{transient total cost MDPs}, with the understanding that our results apply also to the other two types of problems.

Transient total cost MDPs are defined as follows. Consider a partition of $\mathbf{X}$ into sets
$\mathbf{X}'$ and $\mathbf{M}$, with $\mathbf{X} = \mathbf{X}'\cup \mathbf{M}$
and $\mathbf{X}'\cap \mathbf{M} = \emptyset$. 
A policy $\pi$ is said to be transient in $\mathbf{X}'$
if\footnote{$\Pr^\pi_{x_0}[X_t=x]$ is the probability that $X_t=x$ given the initial state
$x_0\in\mathbf X^\prime$ and the policy $\pi$.}
\begin{enumerate}
\item $\sum_{t=0}^\infty \Pr^\pi_{x_0}[X_t=x]< \infty$ for every $x\in \mathbf{X}'$, and
\item  $\mathcal{P}_{yx}^a=0$ for each $y\in \mathbf{M}$, $x\in \mathbf{X'}$, and $a\in A(y)$.
\end{enumerate}
In order words, the transient policy $\pi$ ensures that, \emph{eventually}, the state will enter the \emph{absorbing} set $\mathbf M$.  The second constraint on the transition probabilities
implies that once the state enters $\mathbf{M}$ it will remain there. An MDP for which all
policies are transient in  $\mathbf{X}'$ is called a $\mathbf{X}'$-transient MDP.

We study in this paper the total cost criterion for $\mathbf{X}'$-transient MDPs. We assume throughout the paper that $c(x, a) = 0$  for any $x \in \mathbf{M}$, and that the initial state $x_0 \notin \mathbf{M}$. 
We define  $\sum_{t=0}^{\infty}E_{\pi} [c(X_t,A_t)]$
 as the total expected cost until the set $\mathbf{M}$ is reached 
 \cite{Altman1996},  where the subscript $\pi$ means that the expectation is
with respect to the mass probability induced by the policy $\pi$.
Note that this definition is well posed because
  the $\mathbf{X}'$-transient MDP 
 assumption ensures that the  expected total cost function exists and is bounded.
A transient total cost MDP problem is then defined as follows:
\begin{quote} {\bf Transient total cost MDP} --- Given a $\mathbf{X}'$-transient MDP, determine
a policy $\pi^*$ minimizing the total expected cost, i.e., find  
\[
\pi^* \in \arg \min_{\pi \in \Pi_{\text{D}}} \, \sum_{t=0}^{\infty} \, E_{\pi} [c(X_t,A_t)].
\]
\end{quote}


Optimal policies can be computed in different ways, but, 
in practice, dynamic programming methods (value iteration or policy
iteration) are the  techniques most commonly used.

\subsection{Total Cost Constrained Markov Decision Processes}
\label{subsection:CMDP}
A Constrained Markov Decision Process (CMDP) extends the MDP 
model by introducing additional costs and associated constraints. 
A CMDP is defined by  $\mathbf{X},A,c,\mathcal{P}, \{d_i\}_{i=1}^L$,$\{D_i\}_{i=1}^L$
where $\mathbf{X},A,c,\mathcal{P}$ are the same as above and
furthermore:
\begin{enumerate}
\item $d_i: \mathcal{K}\rightarrow \mathbb{R}_{\geq 0}$, with $1 \leq i \leq L$, is a family of 
$L$ constraint costs incurred when applying action $a\in A(x)$ from state
$x$. We assume that these costs are non-negative.
\item $D_i \in \mathbb{R}_{\geq 0}$ is an upper bound  for the expected cumulative (through time)
 $d_i$  costs. 
\end{enumerate}
Informally, solving a CMDP means determining a policy $\pi$ minimizing the expected objective cost
defined by $c$ while  ensuring that each of the constraint costs defined
by the functions $d_i$ are (in expectation) bounded by $D_i$.
This notion can be formalized as follows. First, it is necessary to
highlight that for CMDPs
an optimal policy in general depends on the probabilistic
distribution of the initial state. In the following this 
distribution is indicated with the letter $\beta$, with the
understanding that $\beta(x) = \Pr[X_0 = x]$ for $x\in \mathbf{X}$. 

The same conditions outlined in the previous subsection to define a transient total cost MDP can be 
used to define a transient total cost CMDP. 
A CMDP for which all policies are transient in  $\mathbf{X}'$ is called a $\mathbf{X}'$-transient CMDP.
As for the total cost MDP problem, we assume that $c(x, a) = 0$, but for the CMDP problem we further assume  that $d_i(x, a) = 0$, $i\in\{1,\ldots,L\}$, 
for any $x \in \mathbf{M}$, and that  $\beta(x)=0$ for all $x \in \mathbf{M}$.

The total expected cost for an $\mathbf{X}'$-transient CMDP is defined as $\sum_{t=0}^\infty E_{\pi,\beta}\left[c(X_t,A_t)\right]$ 
where the expectation is taken with respect to the mass distribution induced by $\pi$ and the initial distribution $\beta$.
Similarly, we can define the total expected constraint costs as $E_{\pi,\beta}\left[d_i(X_t,A_t)\right]$.
Note that because of the assumptions made about the costs in $\mathbf{M}$ and because, by assumption, the CMDP 
is $\mathbf{X}'$-transient, these expectations exist and are finite.

Henceforth, we will use the following notation:
$c(\pi,\beta) := \sum_{t=0}^\infty E_{\pi,\beta}\left [c(X_t,A_t)\right]$,
and $d_{i}(\pi,\beta) :=\sum_{t=0}^\infty E_{\pi,\beta}\left[d_i(X_t,A_t)\right],\,\, 1 \leq i \leq L$. A transient total cost CMDP problem is then defined as follows:
\begin{quote} {\bf Transient total cost CMDP} --- Given a $\mathbf{X}'$-transient MDP, determine a policy $\pi$
minimizing the total expected cost and satisfying the $L$ constraints on the total expected constraint costs, i.e., find
\begin{align}
&\pi^* \in \arg \min_{\pi \in \Pi_{\text{M}}} c(\pi,\beta) \label{COP} \\
&\textrm{s.t.}~d_{i}(\pi,\beta)\leq D_i,~ ~1\leq i\leq L,\nonumber
\end{align}
where $\Pi_{\text{M}}$ is the set of closed-loop, Markovian, stationary, and \emph{randomized} policies (i.e., mapping a state $x$ into a probability mass function over $A(x)$). 
\end{quote}
It is well known (see, e.g., Theorem 2.1 in \cite{Altman1996} and Theorem 6.2 in \cite{altman1999constrained}) that there is no loss of optimality in restricting the attention to policies in $\Pi_{\text{M}}$ (instead, e.g., of also considering history-dependent policies). 
 However, one should not restrict $\pi$ to be in $\Pi_D$, as for the MDP case, because
an optimal policy
might require randomization.
CMDPs do not share many of the properties enjoyed by MDPs 
(see \cite{altman1999constrained} for a comprehensive discussion of the subject.)
For example, CMDPs cannot be solved using dynamic programming but can be 
solved, for example, using a linear programming formulation presented below\footnote{More in general, there exist more than one linear programming formulation that can be used, and methods based on
Lagrange multipliers have been introduced as well. However, they will not be 
considered in this paper.}.

A fundamental theorem concerning CMDPs \cite{Altman1996} relates the solution of
the optimization problem defined in equation \eqref{COP} to the following  linear optimization problem. 
Let $\mathcal{K}' := \{(x,a), x\in \mathbf{X}', a\in A(x)\}$ and 
consider $|\mathcal{K}'|$ optimization variables $\rho(x,a)$, each one associated
with an element in $\mathcal{K}'$. Let $\delta_x(y)=1$ when $x=y$ and 0 otherwise, and define the linear optimization problem:

\begin{align}
\min_{\rho} &\sum_{(x,a) \in\mathcal{K}'} \rho(x,a)c(x,a)\label{LINCOP} \\
\textrm{s.t.}&\sum_{(x,a) \in\mathcal{K}'} \rho(x,a)d_i(x,a)\leq D_i~ ~1\leq i\leq L\nonumber\\
&~ \sum_{y\in \mathbf{X}'} \sum_{a\in A(y)}\rho(y,a)(\delta_x(y)-\mathcal{P}_{yx}^a)=\beta(x)~\forall x\in\mathbf{X}'\nonumber\\ 
&~\rho(x,a)\geq 0 \quad \forall (x,a) \in \mathcal{K}'.\nonumber 
\end{align}
The optimization problem defined in equation  \eqref{COP} has a solution if and only if 
the problem defined in equation \eqref{LINCOP} is feasible \cite{Altman1996}, and the optimal solution to the linear program
induces an optimal, stationary, randomized policy for the CMDP defined as follows:
\begin{equation}
\pi^*(x,a) = \frac{\rho(x,a)}{\sum_{a\in A(x)} \rho(x,a)}~ ~ x\in \mathbf{X}',a\in A(x), \label{eq:induced}
\end{equation}
where $\pi^*(x,a)$ is the probability of taking action $a$ when in state $x$. If the denominator 
in equation (\ref{eq:induced}) is zero, then the policy can be arbitrarily defined for $(x,a)$. 
Note that equation \eqref{eq:induced} does not specify a policy 
 for states in $\mathbf{M}$. Since no further costs are incurred in 
$\mathbf{M}$ and the state cannot leave $\mathbf{M}$ once it enters it, a policy can be arbitrarily defined for those states. 
The optimization variables $\rho(x,a)$ are referred to as {\em occupation measures} \cite{Altman1996}, since for each pair $(x,a)$,
they can be written as:
\begin{equation}
\rho(x,a)=\sum_{t=0}^\infty \Pr[X_t=x,A_t=a],
\label{eq:rhodef}
\end{equation}
where the probability is implicitly conditioned on a policy $\pi$ and an initial distribution $\beta$.
Note that an occupation measure is a sum of probabilities, but in general is \emph{not} a probability itself.

\subsection{Robust Constrained Markov Decision Processes}
\label{sec:RCMDP}

In this section we introduce a natural generalization of the CMDP problem to the case where there is uncertainty in the model's parameters (this uncertainty should not be confused with the probabilistic uncertainty affecting the outcome of a control action). The standing assumption is that the underlying MDP is transient. 
Specifically, we consider the scenario where the stage-wise constraint costs are affected by \emph{additive} uncertainty in the form:  
\[
d_\epsilon(X_t, A_t):=d(X_t, A_t)+\epsilon(X_t, A_t),\,\,\forall (X_t, A_t) \in \mathcal{K}',
\] 
where $d(X_t, A_t)$ represents the \emph{nominal} stage-wise constraint cost, and $\epsilon(X_t, A_t)$ is the uncertain term. Thus, for a \emph{given} sequence of state-action pairs $\{X_t, A_t\}_t$, the \emph{uncertain} constraint cost is defined as: $d_{\epsilon}(\pi,\beta) := \sum_{t=0}^{\infty} \, E_{\pi,\beta} \, [d(X_t,A_t) +\epsilon(X_t, A_t)]$.
 
Note that, according to our definition, the uncertainty is stage-invariant (i.e., it does not change from stage to stage, but of course it is a function of $X_t$ and $A_t$). In practice it is often difficult to give a stochastic characterization of such uncertainty. On the other hand,  in many cases it is still possible to characterize the support of the uncertainties.  Specifically, we let 
$\mathcal U \subset \reals^{| \mathcal{K}'|}$ denote the set of admissible values for the uncertainty. In this paper we restrict our attention to \emph{budgeted interval uncertainty}:
\begin{assumption}[Budgeted Interval Uncertainty]\label{def:uncert}
Let $\{\overline{\epsilon}(x,a)\}_{(x,a)\in \mathcal{K}'}$ be a given non-negative vector in $\reals^{\mathcal{K}'}_{\geq 0}$.  Then the uncertainty set is given by 
\[
\mathcal U = \left\{\epsilon\in\reals^{|\mathcal{K}'|}_{\geq 0}: \begin{array}{l}
0 \leq \epsilon(x,a)\leq \overline{\epsilon}(x,a),\forall (x,a) \in  \mathcal{K}'\\
\sum_{(x,a)\in\mathcal K^\prime} \epsilon(x,a)\leq \Gamma \end{array} \right\},
\]
where $\Gamma$, $0\leq \Gamma \leq \sum_{(x,a)\in\mathcal K^\prime} \overline{\epsilon}(x,a)$, is a given ``uncertainty budget."
\end{assumption}
Note that the constraint costs remain non-negative under any perturbation in $\mathcal U$, hence the problem is well-posed under any realization of the underlying uncertainties, see Section \ref{subsection:CMDP}. Within the Budgeted Interval Uncertainty model, $\epsilon(x,a)$ represents the deviation from the nominal cost coefficient for a given $(x,a)  \in  \mathcal{K}'$. In turn, $\Gamma$ has the interpretation of a \emph{budget of uncertainty} that a
system designer selects in order to easily trade robustness and performance. As an alternative interpretation, it is unlikely that all of the $d(x,a)$, $(x,a)  \in  \mathcal{K}'$, will deviate up to their worst case values, and the goal is to protect the system up to an uncertainty magnitude equal to $\Gamma$. Note a value of $\Gamma=0$ implies that no uncertainty is expected, while a value $\Gamma= \sum_{(x,a)\in\mathcal K^\prime} \overline{\epsilon}(x,a)$ implies that the system designer expects that all parameters will deviate up to their maximum values (in which case the problem could be rewritten as a problem without uncertainty and with modified constraint stage costs $d(x,a) + \overline{\epsilon}(x,a)$). This model of uncertainty is fairly common in robust optimization, we refer the reader to \cite{ben2009robust, DB-MS:03}. Extensions to alternative uncertainty models are possible (e.g., multiplicative uncertainty), and are left for future research.

Consider, then, the \emph{robust} constraint
 \begin{equation}\label{eq_rob_con}
 \sup_{\epsilon\in \mathcal U} \, d_{\epsilon}(\pi,\beta) \leq D.
 \end{equation}
The next lemma (whose proof is provided in the Appendix) shows that we can replace the $\sup$ operator with the $\max$ operator.
\begin{lemma}\label{lemma:sup}
For any policy $\pi$ and initial distribution $\beta$, the supremum in the optimization problem 
$ \sup_{\epsilon\in \mathcal U} \, d_{\epsilon}(\pi,\beta)$
is achieved. Hence one has  
$\sup_{\epsilon\in \mathcal U} \, d_{\epsilon}(\pi,\beta) = \max_{\epsilon\in \mathcal U} \, d_{\epsilon}(\pi,\beta)$.
\end{lemma}

In light of equation \eqref{eq_rob_con} and Lemma \ref{lemma:sup}, we can therefore  formulate the following  Robust Constrained Markov Decision Problem (RCMDP):
\begin{quote} {\bf Optimization problem RCMDP} --- Given a $\mathbf{X}'$-transient CMDP, an initial distribution $\beta$, and a contraint threshold $D$, find
\begin{align}
&\pi^* \in \arg \min_{\pi \in \Pi_{\text{M}}}  c(\pi,\beta) \label{RCMDP} \\
&\textrm{s.t.}~\max_{\epsilon\in \mathcal U} \, d_{\epsilon}(\pi,\beta) \leq D,
\end{align}
over the class of policies $\pi \in \Pi_{\text{M}}$.
\end{quote}

We next show how the  RCMDP problem can be solved as a linear optimization problem on the space of occupation measures. Consider the following optimization problem:
\begin{quote} {\bf Optimization problem $\mathcal{OPT}$} --- For a given initial distribution $\beta$ and risk threshold $D$, solve\begin{alignat*}{2}
\min_{\rho}  & & \quad&  \sum_{(x,a) \in \mathcal{K}'}\rho(x,a)c(x, a)\\  
\text{s.t.} & & \quad&\sum_{y\in \mathbf{X'}}\sum_{a\in A(y)}\rho(y,a)\left[\delta_x(y)- \mathcal{P}_{yx}^a\right]= \beta(x), \forall x\in \mathbf{X'}\\
& &\quad & \max_{\epsilon\in \mathcal U} \sum_{ (x,a) \in \mathcal{K}'}\rho(x,a)(d(x, a)+\epsilon(x,a)) \leq D\\  
& &\quad &\rho(x,a)\geq 0, \quad \forall (x,a) \in \mathcal{K}'.
\end{alignat*}
\end{quote}

Note that problem $\mathcal{OPT}$ is a \emph{robust} linear programming problem. As for non-robust CMDPs, $\rho(x,u)$ has the meaning of occupation measure. The following theorem (whose proof is given in the Appendix) establishes  the relation between RCMDP and $\mathcal{OPT}$. 
\begin{theorem}\label{thm:OM}
The RCMDP problem has a solution if and only if problem $\mathcal{OPT}$
 is feasible. The optimal solution to the robust linear program $\mathcal{OPT}$
induces an optimal stationary, randomized policy for RCMDP defined as follows:
\begin{equation}
\pi^*(x,a) = \frac{\rho^*(x,a)}{\sum_{a\in A(x)} \rho^*(x,a)},~ ~ x\in \mathbf{X}',a\in A(x), \label{eq:rob_induced}
\end{equation}
where $\{\rho^*(x,a)\}_{(x,a)\in  \mathcal{K}'}$ is the optimal solution to problem $\mathcal{OPT}$. For states $x\in\mathbf X$ such that $\sum_{a\in A(x)} \rho^*(x,a)=0$, $\pi^*(x,a)$ is arbitrarily chosen, for every $a\in A(x)$.
\end{theorem}

\section{The Deployment Problem as a RCMDP}
\label{sec:formulation}
In this section we formalize the rapid deployment problem we wish to study and we show
how it can be modeled as an instance of a RCMDP. 

\subsection{General Description}\label{subsec:gendes}
At a high level, the problem is as follows. A set of robots  are placed within a bounded environment where a set of points represent target locations. 
The objective is to  deploy the robots  so that each location is 
reached by at least one robot and the deployment task takes no longer than a given temporal deadline. However, the environment is stochastic, in the sense that robots might ``probabilistically" fail. In such a stochastic setup, the objective is then to maximize the \emph{probability} that each location is 
reached by at least one robot while ensuring that the \emph{expected} duration of the deployment task is upper bounded by a given temporal deadline.  In this paper we interpret the \emph{duration} of the deployment task as the elapsed time between the common instant when robots start moving and the instant when the \emph{last} robot stops moving, either because of a failure \emph{or} because a target has been reached.

Specifically, following the approach presented in \cite{CarpinICRA2013}, the environment is abstracted into an undirected graph $G=(X,E)$
where $X$ is the set of vertices and $E$ is the set of edges. (The vertex set $X$ will be later mapped into the state space $\mathbf{X}$ in the ensuing RCMDP model.) An edge $e\in E$  between two vertices $v_1$ and $v_2$ means that a robot can move from $v_1$ to $v_2$ and vice versa. We consider that there are $K$ robots, initially located at a single vertex $v_0\in X$, that are required to reach a set of target vertices, denoted as $T\subset X$. Self-loops in $G$ are only allowed for vertices in the target set, i.e., $(v,v)\in E$ if and only if $v\in T$. The deployment task is considered successful if (i) each vertex in  $T$ is reached by at least one robot, and (ii) this is accomplished within a given temporal deadline $D$.

To capture the time/safety trade-off, we associate to each edge $e_i \in E$, $i\in|E|$, a \emph{safety function} $S_e: \mathbb{R}_{\geq 0}\rightarrow [0,1]$. Function $S_e(t)$, for each $t$, represents  the probability of successfully traversing an edge $e$ given that the traversal time is equal to $t$. The safety function satisfies the boundary constraints $S_e(0)=0$ and $\lim_{t\rightarrow +\infty} S_e(t) = 1$, and must be non-decreasing. Other than those constraints, the safety function is arbitrary (and potentially different for each edge). An example is provided in Figure \ref{fig:sfunction}, where the probability of success is modeled using a sigmoidal function. There are many other success probability models available in the survival analysis community. For example, in \cite{rausand2004system}, the hazard probability functions are modeled either with Exponential, Gamma and Weibull distributions or with other non-parametric estimates. However the advantages of modeling the success probability with a sigmoidal/logit function are, (i) this function serves as a good model for binary classification, and (ii) the model parameters can be easily estimated by logistic regression, after experimental data is given. 


Accordingly, the deployment problem entails finding control policies for each robot that prescribe at each vertex (i) which vertex to visit next, and (ii) at what speed the edge should be traversed. Robots should move slowly to increase their chances of successfully traversing the edges of the graph, however they can not arbitrarily slow down because they will otherwise miss the temporal deadline. In this paper we consider robotic swarms, hence we seek \emph{minimalistic} control strategies that do not involve any communication after the deployment process is initiated.
\begin{figure}[tbh]
\centering
\includegraphics[width=0.6\linewidth]{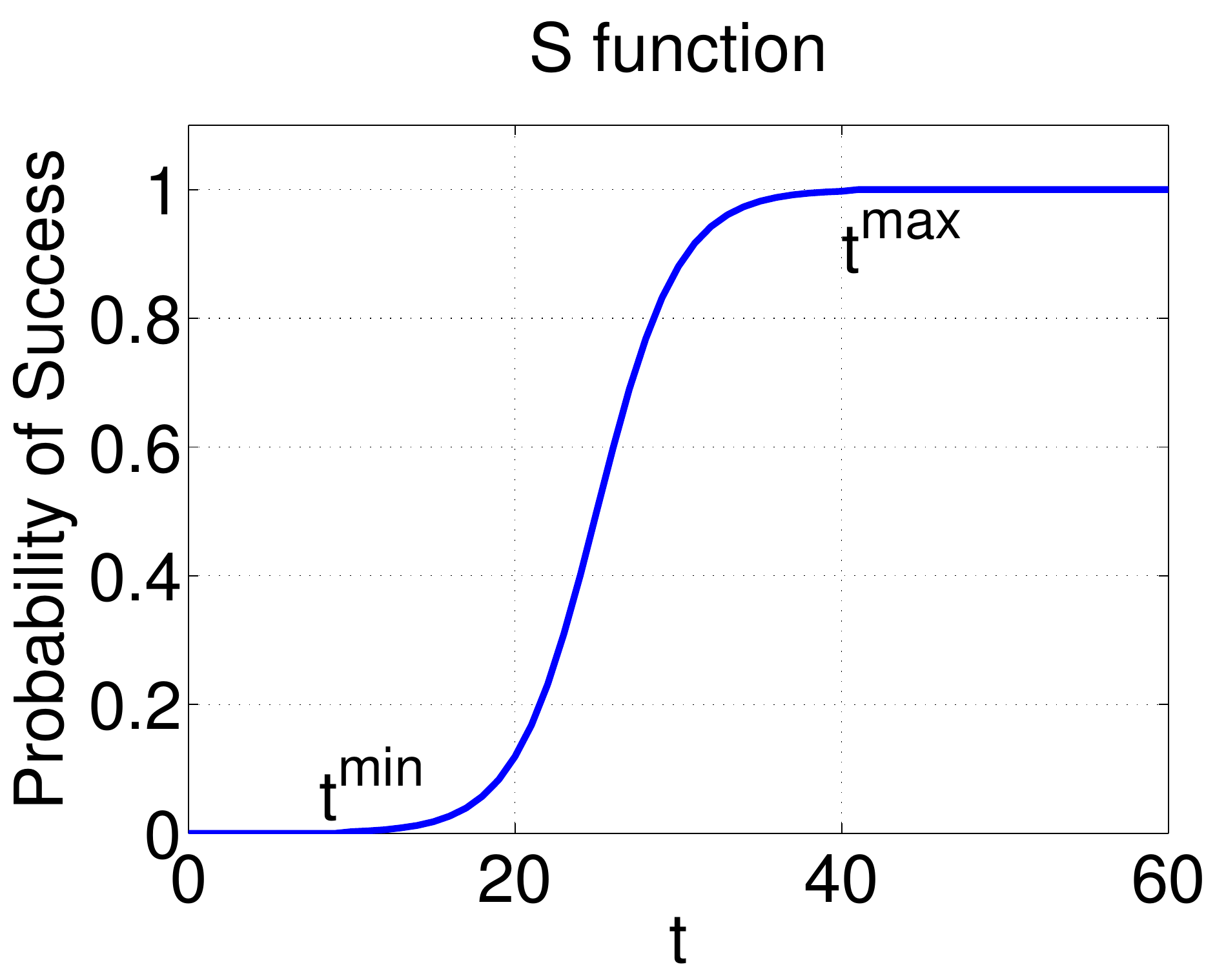}
\caption{A sigmoidal shape for the safety  function $S_e$ associated with the edges in the graph.}
\label{fig:sfunction}
\end{figure}
\subsection{Formulation as RCMDP --- Single-Robot}\label{subsec:singlerobot}
We now show how to cast the aforementioned deployment problem within the RCMDP model. 

Consider, first, the single-robot case and, hence, a target set that comprises a single vertex (i.e., $|T|=1$). The state space comprises all the vertices of the graph $G$ plus a virtual \emph{sink state} $\mathcal{S}$ modeling failures, i.e., the robot enters the sink state when it experiences an irremediable  failure while traversing an edge and then stops functioning. The state space is then  $\mathbf{X}:=X\cup \{\mathcal{S} \}$.
The initial distribution $\beta$ is defined as $\beta(v_0)=1$ and $\beta(x)=0$ for $x\in X, x \neq v_0$, since, in our model, $v_0$ is (deterministically) the initial location of the robot. At each vertex $y\in X$, the robot needs to decide (i) what vertex to visit next, and (ii) the traversal time (or, equivalently, the traversal speed), within given bounds (dictated by velocity bounds). The traversal time is a continuous variable, which is discretized with a time step $\Delta>0$ to make the model amenable to dynamic optimization. (For simplicity, we assume that the discretization step $\Delta$ is equal for all stages.) Summarizing, the action set at a vertex $y\in X, y\notin T$, is given by:

\begin{equation*}
\begin{split}
A(y) = \Bigl \{ (x,t)\in X\times \reals_{\geq 0}: &(y,x)\in E,   \, \, t = t^{\mathrm{min}}_{yx} + k\, \Delta,\\ k\in \naturals, \, \, 
&0\leq k\leq \Bigl \lfloor \frac{t^{\mathrm{max}}_{yx}- t^{\mathrm{min}}_{yx}}{\Delta}\Bigr \rfloor \Bigr \},
\end{split}
\end{equation*}
where $t^{\mathrm{min}}_{yx}\geq 0$ and $t^{\mathrm{max}}_{yx}>0$ are the minimum and maximum, respectively, traversal times for the edge $(y, x)\in E$ (see Figure \ref{fig:sfunction}). An action $(x,t) \in A(y)$ means that the robot decides to navigate from $y$ to $x$ spending
 time $t$. For the target vertex, i.e., $y\in T$, we consider a single action:
 \begin{equation*}
\begin{split}
A(y) = \Bigl \{ (x,t)\in X\times \reals_{\geq 0}: x = y, \, \,  t = t_{yy}\Bigr \},
\end{split}
\end{equation*}
 where $t_{yy}>0$ models the self-loop time (the choice of this value is immaterial, since $T$ will be the absorbing set in the RCMDP model). For the special sink state $\mathcal{S}$ we define just one action, denoted as $a_{\mathcal{S}}$, which lets the state of the robot transition to the unique vertex in the target set $T$ (this choice is made to ensure that $T$ is the absorbing set in the ensuing RCMDP model). 
 
The above action sets and the fact that the traversal of each edge in $E$ can only be accomplished probabilistically (according to the safety function $S_e$) induce the following transition probabilities on $\mathbf{X}$:

\[
\mathcal{P}_{yx}^a =
\begin{cases}
S_{(y,x)}(t) & \textrm{if } y\neq \mathcal{S}, \, a=(x,t) \in A(y), \, x\neq \mathcal{S}, \\ 
1-S_{(y,x)}(t) & \textrm{if } y\neq \mathcal{S}, \, a=(x,t)\in A(y), \, x= \mathcal{S},\\
1 & \textrm{if } y= \mathcal{S}, \, a=a_{\mathcal{S}}, \, x\in T.\\ 
\end{cases}
\]
All other transition probabilities are equal to zero. The first case models the probability of successfully completing a traversal from $y$ to $x$, where the success probability is dictated by the safety function $S_{(y,x)}$. The second case
models the failure probability, i.e., when a robot does not 
successfully perform the traversal from $y$ to $x$ and, hence,  it enters the sink state $\mathcal{S}$. The third case enforces the deterministic transition from a sink state to the target vertex (which, as discussed before, ensures the $T$ is the absorbing set). It is immediate to conclude that, with these transition probabilities, any Markovian, stationary, and randomized policy on $\mathbf{X}$ is transient on the set $\{\mathbf{X}\setminus T\}:=\mathbf{X^\prime}$; the corresponding absorbing set is $\mathbf{M}:=T$. Figure \ref{fig:SinkState} illustrates the relationship between $\mathbf{X}$,
$\mathcal{S}$, and $T$.

\begin{figure}[tbh]
\centering
\includegraphics[width=0.7\linewidth]{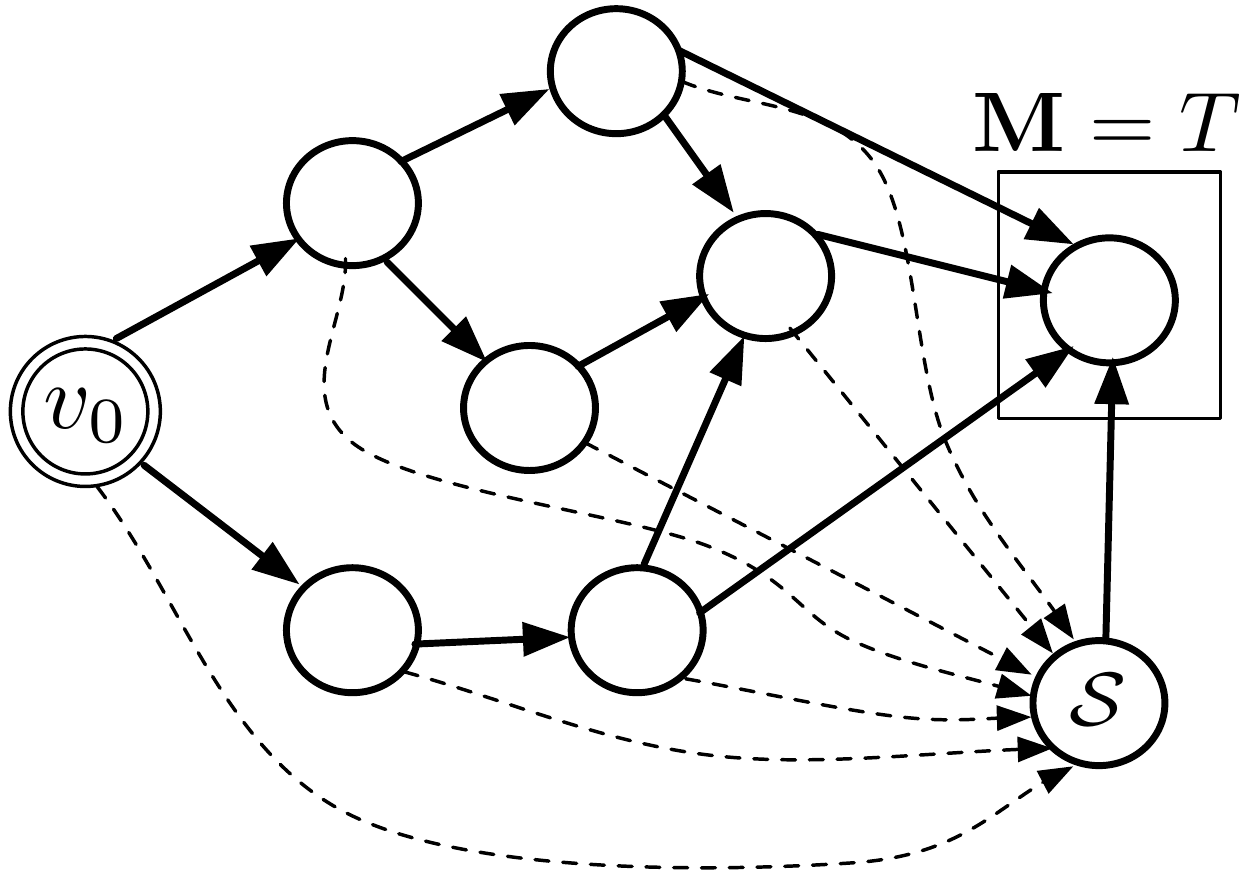}
\caption{Given a graph $G=(X,E)$ and a policy $\pi$, multiple stochastic
paths from the deployment vertex $v_0$ to the target vertex set $T$
 exist. Whenever a failure occurs, the state  enters $\mathcal{S}$ (dashed arrows). States outside the
box labeled $\mathbf{M}$ are in $\mathbf{X}'$.
}
\label{fig:SinkState}
\end{figure}

To complete the definition of the RCMDP model,  we need to define the objective costs $c(x,a)$ and
the constraint costs $d(x,a)$. 
Given that our stated objective is to maximize the probability of success (or equivalently, minimize the probability of failure), the objective costs are defined as follows:
\[
c(y,a) =
\begin{cases}
0& \textrm{if } y\neq \mathcal{S}, a=(x,t) \in A(y), \\
1& \textrm{if } y = \mathcal{S}, \, a=a_{\mathcal{S}}.
\end{cases}
\]
Note that, with this choice, $\sum_{t=0}^\infty E_{\pi,\beta}\left [c(X_t,A_t)\right]$ is exactly the probability that the robot fails to reach the target. To see this, denote the sample space by $\Omega$, define the events $B_t=\{X_t = \mathcal S \}$, for $t=0,1,\ldots$, and let $B:= \Omega\setminus  \cup_{0}^{\infty}\, B_t$. Clearly, the events $\{B_t\}_{t=0}^{\infty}$ and $B$ are collectively exhaustive, i.e., their union yields $\Omega$. Also,  there is only one action available when the system is in 
 $\mathcal{S}$, and such action deterministically moves the state to the absorbing set $\mathbf{M}$ (where the state remains trapped). Hence, the events $\{B_t\}_{t=0}^{\infty}$ and, of course, $B$, are mutually exclusive. We can then apply the law of total probability and write (where we omit the subscripts $\pi$ and $\beta$ for simplicity)
\begin{equation}\label{eq:fail_prob}
\begin{split}
\Pr[\mathrm{Failure}] &= \sum_{t=0}^{\infty}\, \underbrace{\Pr[F|B_t]}_{=1}\Pr[B_t] + \underbrace{\Pr[F|B]}_{=0}\Pr[B]\\
&= \sum_{t=0}^{\infty} \Pr[X_t=\mathcal S] = \sum_{t=0}^{\infty} \Pr[X_t=\mathcal S, A_t=a_{\mathcal{S}}] \\
&= \sum_{t=0}^\infty E\left [c(X_t,A_t)\right],
\end{split}
\end{equation}
and the claim follows. Also, the above derivation highlights that $\sum_{t=0}^\infty E_{\pi,\beta}\left [c(X_t,A_t)\right]$ = $\rho(\mathcal S, a_{\mathcal S})$, hence the cost function is indeed equal to the occupation measure for the state-action pair $(\mathcal S, a_{\mathcal S})$.

Similarly, the constraint costs are defined as 
\[
d(y,a) =
\begin{cases}
t  & \textrm{if } y\in X\setminus T, a=(x,t) \in A(y), \\
0  & \textrm{if } y\in T,\, \,  a=(x,t) \in A(y), \\
0 & \textrm{if } y = \mathcal{S}, \, \, a=a_{\mathcal{S}}.
\end{cases}
\]
Note that, for simplicity, travel times are assumed deterministic, but our framework can be easily extended to the case where travel times are stochastic. With these definitions both the objective and the constraint costs evaluate to zero on the absorbing set $\mathbf{M} = T$ (in particular, $d(y,a)$ for $y\in T$ is set to zero rather than to $t_{yy}$). 

In this paper we assume that the constraint costs (i.e., the travel times) are not exactly known  (as it is oftentimes the case in practical scenarios), and they are instead characterized through a budgeted interval uncertainty model (see Assumption \ref{def:uncert}). Conceptually, the traversal time $t$ selected by a robot should be interpreted as an \emph{intended} travel time: the (uncertain) constraint costs then translate intended travel times into \emph{actual} travel times, whereas safety functions  translate  intended travel times into risk of failure (the idea being that, for example, shorter intended travel times correspond to more dangerous maneuvers). Note that within our model the transition probabilities are \emph{certain} (the uncertainty in the mapping from intended travel times to risk is encapsulated in the safety function $S_e$). One could also consider formulations where the transition probabilities and the objective costs are uncertain as well: this, however, would significant complicate the model and is left for future research.

The single-robot deployment problem is then reformulated as a RCMDP problem: find a Markovian policy that minimizes the probability of failure (i.e., the summation of the expected objectives costs $c$) while \emph{robustly} keeping the traversal time (i.e., the summation of the expected constraint costs $d$) below a given temporal deadline $D$. The robustness of the formulation stems from the fact that we consider traversal times (i.e., the functions $d(y,a)$) that are uncertain.

This formulation is consistent  with our definition of task duration as discussed in Section \ref{subsec:gendes}. Other formulations are possible, for example one might be interested in minimizing the expected deployment duration while constraining the probability of failure below a given threshold (possibly with uncertain safety functions), or might desire to constrain the duration of the deployment task only for those executions that do not involve failures. Such formulations are left for future research. We mention, however, that in many cases the latter formulation and our formulation are essentially equivalent. This can be seen by observing that $E[\mathrm{task \, duration}|\mathrm{success}] \Pr(\mathrm{success})\,\leq E[\mathrm{task \, duration}] \leq E[\mathrm{task \, duration}|\mathrm{success}] $, where the first inequality is a consequence of the law of total expectation and the second inequality follows from the fact that in case of failure the state enters $T$ through a {\em shortcut} via $\mathcal{S}$. Assuming that  $\Pr(\mathrm{success})$ (which can be computed exactly in our formulation) is high, say, 90\%, the discrepancy between the two formulations will then be small. We investigate this aspect further in Section  \ref{sec:simulation}.

\subsection{Formulation as RCMDP -- Multi-Robot}\label{sec:RCMDP_multi}
The multi-robot formulation is indeed a simple extension of the single-robot formulation. The main tenet in our approach is that we seek minimalistic control strategies that do not involve any communication once the robotic swarm is deployed (which is oftentimes a requirement for the deployment of robotic swarms comprising simple and inexpensive platforms). We note that the proposed model can also serve as a yardstick for comparison with more sophisticated coordination mechanisms, by providing easily-computable bounds for the achievable safety-speed Pareto curves. Our key technical assumption is that  congestion effects (i.e., robots colliding into each other) are  negligible. In other words, we assume that failures are \emph{statistically independent}. This is normally the case in the fast growing domain of  robotic swarms comprising minimalistic, palm-sized micro-aerial vehicles (MAVs) \cite{Purohit2011a,Kumar01092012,AK:DM:12,DM:NM:12}.

Consider $K\geq |T|$ robots (if $K< |T|$, clearly, the deployment objective as formulated before can not be accomplished). Our formulation of the multi-robot deployment problem essentially decouples the target assignment problem from the path planning problem, specifically:
\begin{enumerate}
\item each robot is assigned a deployment location in $T$;
\item each robot executes a single-robot deployment policy (by solving the associate RCMDP) to reach the assigned location.
\end{enumerate} 
Such formulation fulfills the requirement that robots do not need to communicate during the deployment process (e.g.,
to learn how many robots are in the team, or to communicate that a certain target vertex has been already reached, or
to update the safety functions $S_e$). The problem, then, is how to assign targets to robots. Let $\alpha:=(\alpha_1, ...\ldots, \alpha_{K})\in T^K$ denote a target assignment, with the understanding that $\alpha_i = v\in T$ implies that the $i$th robot is assigned to target $v$, $i=\{1, \ldots, K\}$. Since each robot executes a  single-robot deployment policy (in a statistically independent fashion, by assumption), the probability that the overall deployment is successful, given an assignment $\alpha$, can be easily computed from the probabilities that each robot can successfully reach its assigned target. Specifically, let $\varphi(\alpha)$ denote the probability that the deployment task is successful for a given assignment $\alpha$, and let $\mathrm{PF}(v)$, for $v\in T$, denote the probability that a robot assigned to target $v$ fails to reach such target (these probabilities are simply the optimal costs of the corresponding RCMDPs). The probability $\varphi(\alpha)$ is then given by
\[
\varphi(\alpha)  = \prod_{j=1}^{|T|}\, \Bigl(1 - \mathrm{PF}(v_j)^{|\{\alpha_i\in \alpha:\, \, \alpha_i=v_j\}|} \Bigr),
\]
where $v_j$ is the $j$th target in $T$. Accordingly, the target assignment problem becomes:
\begin{quote} {\bf Target Assignment (TA)} --- For $K\geq|T|$, solve
\begin{align*}
&\max_{k_j} \quad \prod_{j=1}^{|T|}\, \Bigl(1 -\mathrm{PF}(v_j)^{k_j+1} \Bigr)\\
&\textrm{s.t.}\qquad \sum_{j=1}^{|T|} \, k_j = K-|T|\\
&\quad \qquad k_j\geq 0, \, \,k_j\in\naturals,\,\, j\in \{1,\ldots, |T|\}.
\end{align*}
\end{quote}
The formulation assigns target $v_1\in T$ to the robots 1 to $k_1 +1$,  target $v_2$ to robots $k_1+2$ to $k_1 + k_2 +2$ and so on, in order to minimize the global success probability. The rapid multi-robot deployment problem is then formulated as follows (assuming $K\geq |T|$):
\begin{enumerate}
\item For each target in $v\in T$, solve the RCMDP version of the single-robot deployment problem with target set equal to $v$ (note that the optimal cost of the RCMDP is exactly the probability $\mathrm{PF}(v)$).
\item Solve the target assignment problem  TA and assign targets to robots accordingly.
\item Let each robot execute the single-robot optimal deployment policy to reach its assigned target.
\end{enumerate} 

This formulation requires computationally-efficient solutions for RCMDP and TA, which are discussed next.
 
\section{Efficient Solution for RCMDP}\label{sec:dp}

As shown in the proof of Theorem \ref{thm:OM} (provided in the Appendix), the RCMDP problem can be solved as a robust linear optimization problem, and in particular as a linear optimization problem with a number of constraints equal to the number of vertices of the uncertainty set $\mathcal U$. Since this number is, possibly, exponential\footnote{The number of vertices indeed depends on the value of $\Gamma$. In the extreme case where $\Gamma = 0$ the uncertainty set $\mathcal U$ has only \emph{one} vertex.} in the number of states and actions, it becomes critical  to find an efficient algorithm for the solution of problem $\mathcal{OPT}$.  We next show how problem $\mathcal{OPT}$ can be transformed into a linear programming problem with a number of constraints and variables that is \emph{linear} in $|\mathcal{K}'|$  for \emph{all} choices of the uncertainty budget $\Gamma$ (the strategy is to introduce an auxiliary set of variables whose cardinality is $ |\mathcal{K}'|+1$). Consider the following optimization problem:
\begin{quote} {\bf Optimization problem $\mathcal{OPT}_2$} --- Given an initial distribution $\beta$ and a risk threshold $D$, solve\begin{alignat*}{2}
\min_{\rho, \, \lambda,\, \mu} \quad \!\!\!\!& & \quad& \sum_{(x,a) \in \mathcal{K}'}\rho(x,a)c(x,a)\\  
\text{s.t.} & & \quad&\sum_{y\in \mathbf{X}'}\sum_{a\in A(y)}\rho(y,a)\left[\delta_x(y)- \mathcal{P}_{yx}^a\right]= \beta(x), \forall x\in \mathbf{X}'\\
& &\quad&\sum_{(x,a) \in \mathcal{K}'}\rho(x,a)d(x,a)+\overline{\epsilon}(x,a) \lambda(x,a) + \mu \, \Gamma\leq D\\   
& & \quad& \lambda(x,a) + \mu\geq \rho(x,a),\quad\forall (x,a) \in \mathcal{K}'\\
& &\quad&\rho(x,a), \, \lambda(x,a) \geq 0,\,\,\forall (x,a) \in \mathcal{K}'\\
& &\quad&\mu\geq 0.
\end{alignat*}
\end{quote}
Note that optimization problem $\mathcal{OPT}_2$ involves $2 |\mathcal{K}'|+1$ decision variables ($\mu$ is just a scalar).

\begin{theorem}\label{thm:LP}
Let $\{\rho^*(x,a)\}_{(x,a) \in \mathcal{K}'}$ be part of the optimal solution of problem $\mathcal{OPT}_2$. Then, $\{\rho^*(x,a)\}_{(x,a) \in \mathcal{K}'}$ is an optimal solution to problem $\mathcal{OPT}$.
\end{theorem}
\begin{proof}
For any fixed $\rho:=\{\rho(x,a)\}_{(x,a) \in \mathcal{K}'}$, and $D\in\reals_{\geq0}$, consider the constraint:
\begin{equation}\label{constraint_1_RS}
\max_{\epsilon\in \mathcal U} \sum_{(x,a) \in \mathcal{K}'}\rho(x,a)(d(x, a)+\epsilon(x,a)) \leq D,
\end{equation}
which can be written equivalently as 
\[
\max_{\epsilon\in \mathcal U} \sum_{(x,a) \in \mathcal{K}'}\rho(x,a)\epsilon(x,a) \leq D-\sum_{(x,a) \in \mathcal{K}'}\rho(x,a)d(x, a).
\]
The optimization problem on the left-hand side of the above equation can be written explicitly as:
\begin{alignat*}{2}
p^\ast(\rho)=\max_{\epsilon} & & \quad& \sum_{(x,a) \in \mathcal{K}'}\rho(x,a)\epsilon(x,a)\\  
\text{s.t.} & & \quad&0\leq \epsilon(x,a)\leq \overline{\epsilon}(x,a),\quad \forall  (x,a) \in \mathcal{K}' \\
& & \quad &\sum_{(x,a)\in\mathcal K^\prime} \epsilon(x,a)\leq \Gamma.
\end{alignat*}

Consider the dual formulation of the above linear programming problem: 
 \begin{alignat*}{2}
d^\ast(\rho)=\min_{\lambda, \, \mu} & & \quad&\!\!\!\!\!\!\sum_{(x,a) \in \mathcal{K}'}\overline{\epsilon}(x,a) \lambda(x,a)  +\mu \,  \Gamma\\  
\text{s.t.} & & \quad& \lambda(x,a) + \mu\geq \rho(x,a),\quad\forall (x,a) \in \mathcal{K}'\\
& &\quad&\lambda(x,a) \geq 0,\quad \forall(x,a)\in \mathcal{K}'\\
& &\quad & \mu \geq 0.
\end{alignat*}
By strong duality in linear programming, the primal optimal cost is equal to the dual optimal cost. Thus, one has $p^\ast(\rho)=d^\ast(\rho)$. The constraint in expression (\ref{constraint_1_RS}) can then be equivalently written as
\begin{equation}\label{eq:constraint}
\begin{split}
&\exists \,\lambda,\, \mu:\\
&\sum_{(x,a) \in \mathcal{K}'}\rho(x,a)d(x,a)+\overline{\epsilon}(x,a) \lambda(x,a) +\mu\, \Gamma\leq D\\  
& \lambda(x,a) + \mu \geq \rho(x,a),\quad \forall (x,a) \in \mathcal{K}'\\
& \lambda(x,a) \geq 0,\quad \forall (x,a) \in \mathcal{K}'\\
& \mu\geq 0.
\end{split}
\end{equation}
This implies that the robust constraint  in problem $\mathcal{OPT}$ can be replaced by the set of linear constraints in expression (\ref{eq:constraint}). The claim then follows.
\end{proof}
Since the number of constraints in the dual formulation is always $O(|\mathcal K^\prime|)$, the computational complexity of  problem $\mathcal{OPT}_2$ (assuming one uses an interior point algorithm for its solution) is linear in the number of state-action pairs  \cite{YN:AN-94}, for \emph{all} choices of the uncertainty budget $\Gamma$.

\section{Efficient Solution of the Target Assignment Problem}\label{sec:Efficient_TA}
In this section we show how to efficiently solve the TA problem introduced in Section \ref{sec:RCMDP_multi}.
In the following, we assume that $\mathrm{PF}(v_j) \in (0,1)$ for all $j$. This is without loss of generality. In fact, if $\mathrm{PF}(v_j) =1$ for some $j$, the deployment problem is infeasible (since it is impossible to reach target $j$ no matter how many robots are destined there), while if $\mathrm{PF}(v_j) =0$ for some $j$, then one should just send one robot to target $j$ and consider the reduced problem without target $j$. 

Consider, first, the case where $K\geq |T|$ and $K<2|T|$. This is the case where the number of robots is on the order of the number of target locations. In practice, the number of elements in $T$ is usually no larger than a small constant, say, 50, and in this case problem TA can be efficiently solved by using branch and bound techniques.

Let us consider the more challenging case where $K\geq 2|T|$ (the usual case for robotic swarms). In the remainder of this section we present a polynomial-time, asymptotically-optimal algorithm for the solution of problem $\mathrm{TA}$.  To this purpose, consider the following relaxed version of problem $\mathrm{TA}$:
\begin{quote} {\bf Relaxed Target Assignment (RTA)} ---  For $K\geq2|T|$, solve
\begin{align*}
&\max_{k_j} \quad \sum_{j=1}^{|T|}\log\left(1 -\mathrm{PF}(v_j)^{k_j+1} \right)\\
&\textrm{s.t.}\qquad \sum_{j=1}^{|T|} \, k_j = K-2|T|,
\end{align*}
\end{quote}
where we have taken the $\log$ of the objective function, we have relaxed the integrality and non-negativity constraints, and we have reduced the right hand side of the equality constraint to $K-2|T|$. Note that problem RTA is a concave maximization problem with a linear equality constraint. The first order sufficient optimality condition reads as \cite{LuenbergerBook}:
\begin{equation}\label{eq:lambda}
\lambda^\ast=\frac{\mathrm{PF}(v_j)^{k^\ast_j+1}\log(\mathrm{PF}(v_j))}{1-\mathrm{PF}(v_j)^{k^\ast_j+1}},\,\text{for all } j\in\{1,\ldots,|T|\},
\end{equation}
which implies:
\[
k^\ast_j=\log\left(\frac{\lambda^\ast}{\lambda^\ast+\log(\mathrm{PF}(v_j))}\right)\frac{1}{\log(\mathrm{PF}(v_j))}-1,
\]
for all  $j\in\{1,\ldots,|T|\}$. Note that since $\mathrm{PF}(v_j) \in (0,1)$, equation \eqref{eq:lambda} implies $\lambda^*\leq 0$. The Lagrangian multiplier is found by solving the primal feasibility equation, that is:
\[
\sum_{j=1}^{|T|}\log\left(\frac{\lambda^\ast}{\lambda^\ast+\log(\mathrm{PF}(v_j))}\right)\frac{1}{\log(\mathrm{PF}(v_j))}=K - 2|T|,
\]
which can be readily solved, for example, by using the bisection method.

Consider the approximation algorithm for problem TA in  Algorithm \ref{alg:algorithm_different_s}. We next prove that the above approximation algorithm is asymptotically optimal, that is it provides a feasible solution for problem TA whose cost converges to the optimal cost as $K\to +\infty$.
\begin{algorithm}
\caption{Approximation algorithm for problem TA}
\label{alg:algorithm_different_s}
\algsetup{linenodelimiter=}
\begin{algorithmic}[1]
\STATE Solve problem RTA and obtain $\lambda^*$ and $k_j^*$ for all $j\in \{1, \ldots, |T|\}$\;
\STATE Pick (arbitrarily) a set of \emph{non-negative} integers $r_1, r_2, \ldots, r_{|T|}$ such that $\sum_{j=1}^{|T|} r_j= K - |T| - \sum_{j=1}^{|T|}\lceil k_j^*\rceil$ (note that $ \sum_{j=1}^{|T|}\lceil k_j^*\rceil \leq K- |T|$, so this is always possible)\;
\STATE $\bar k_j \leftarrow \lceil k_j^*\rceil + r_j$\;
\STATE Return $\{\bar k_j\}_j$ as approximate solution for problem TA\;
\end{algorithmic}
\end{algorithm}

\begin{theorem}\label{thm:analysis_approx}
The approximation algorithm for problem TA:
\begin{enumerate}
\item delivers a \emph{feasible} solution for all values of $K\geq 2|T|$,
\item is asymptotically optimal.
\end{enumerate} 
\end{theorem}
\begin{proof}
First, note that the objective function in problem RTA acts as a barrier function, hence its optimal solution, $\{k^*_j\}_j$, satisfies the condition $k^*_j>-1$ for all $j\in\{1,\ldots, |T|\}$. This implies that the solution provided by the approximation algorithm, $\{\bar k_j\}$, is a feasible solution for problem TA, in fact:
\begin{enumerate}
\item $\bar k_j\geq 0$ (since $k^*_j>-1$),
\item $\sum_{j=1}^{|T|}\bar k_j  = K-|T|$ (by construction).
\end{enumerate}
This proves the first part of the claim. To prove the second part of the claim, consider the approximation error, $e$, in terms of the log of the objective function of problem TA. The approximation error can be bounded as:
\begin{equation}
\begin{split}
|e| \!&= \!\Biggl |\sum_{j=1}^{|T|}\log\left(1 \!-\!\mathrm{PF}(v_j)^{k^{**}_j+1}\right) \! -\!  \sum_{j=1}^{|T|}\log\left(1 \!-\!\mathrm{PF}(v_j)^{\bar k_j+1} \right)\Biggr| \\
&\leq\sum_{j=1}^{|T|} \Biggl|\log\frac{\left(1 -\mathrm{PF}(v_j)^{k^{**}_j+1}\right)}{\left(1 -\mathrm{PF}(v_j)^{\bar k_j+1} \right)}  \Biggr|,
\end{split}
\end{equation}
where $\{k^{**}_j\}_j$ is the optimal solution to problem TA (which should not be confused with the optimal solution to problem RTA, that is, $\{k^*_j\}_j$). Hence, one can bound $e$ according  to
\[
e \leq T\, \max_{j}  \Biggl|\log\frac{\left(1 -\mathrm{PF}(v_j)^{k^{**}_j+1}\right)}{\left(1 -\mathrm{PF}(v_j)^{\lceil k_j^*\rceil + r_j+1} \right)}  \Biggr|.
\]
Note that \emph{all} $\{k_j^{**}\}_j$'s and $\{k_j^{*}\}_j$'s  tend to $+\infty$ as $K\to \infty$ (this can be easily shown by noticing that if some of the $k_j$ variables stay uniformly bounded as $K\to +\infty$, for $K$ large enough one obtains a suboptimal solution --- this exploits the assumption that $\mathrm{PF}(v_j)\in (0,1)$). Hence, as $K\to +\infty$, one obtains $|e|\to 0$.
\end{proof}

In summary, when $K\leq 2|T|$, we directly solve the target assignment problem using branch and bound.  When $K\geq 2|T|$ the solution can instead be (approximately) found by solving a convex problem.

\section{Solution Algorithm for Rapid Swarm Deployment}\label{sec:swarm_algo}

Collecting the results so far, we can now present a computationally-efficient algorithm for the solution of the rapid multi-robot deployment problem.  The pseudocode is presented in Algorithm \ref{alg:algorithm_different_s2}.

\begin{algorithm}
\caption{Deployment Algorithm}
\label{alg:algorithm_different_s2}
\algsetup{linenodelimiter=}
\begin{algorithmic}[1]
\STATE $\alpha\leftarrow$ solution (possibly approximate)  to target assignment problem TA\;
\STATE Assign target $\alpha_k$ to each robot $k \in \{1, \ldots, K\}$\;
\FOR{\textbf{each} $k$}
\STATE Build  RCMDP instance  with $\mathbf{M}=\{v_{\alpha_k}\}$\;
\STATE $(\rho^*, \lambda^*, \mu^*) \leftarrow$ Solve Problem $\mathcal{OPT}_2$ 
\FOR{\textbf{each}  $x\in \mathbf{X}',a\in A(x) $}
\IF{$\sum_{a\in A(x)} \rho^*(x,a)>0$} 
\STATE $\pi^*(x,a) \leftarrow {\rho^*(x,a)}/{\left(\sum_{a\in A(x)} \rho^*(x,a)\right)}$ 
\ELSE
\STATE Choose an arbitrary value  for $\pi^*(x,a)$
\ENDIF
\ENDFOR
\STATE Navigate to $v_{\alpha_k}$ according to policy $\pi^*$
\ENDFOR
\end{algorithmic}
\end{algorithm}

In case no a-priori coordination is possible, one should consider as target assignment strategy a random uniform choice of the targets (made in a distributed fashion by the robots). 


\section{Numerical Experiments and Discussion}
\label{sec:simulation}

We present three sets of numerical experiments. In the first set of experiments we focus on the single-robot deployment problem, for a fixed uncertainty budget. The objective is to experimentally verify the correctness of the results obtained for problem RCMDP. In the second set of experiments we investigate the sensitivity of the solution for the single-robot deployment problem with respect to different  values of the uncertainty budget (in other words, we investigate how a system designer can trade robustness and performance). In the third set of experiments, we consider the swarm deployment problem, and we discuss how the success probability scales with the number of robots for both a random target assignment algorithm (which requires no coordination among the robots), and the optimized target assignment algorithm (which requires a priori coordination among the robots) discussed in Section \ref{sec:Efficient_TA}.

For all numerical experiments we consider a map that is obtained from
the publicly available {\em Radish} dataset (see Figure \ref{fig:maps}), where the traversal graph used for the simulations is overlaid onto the blueprint of the environment. 
For lack of space we present results concerning this map only, but our results are representative for a wide range of 
environments.
Target vertices are marked with a number and the initial vertex is
indicated with a pink triangle (vertex number 1). Each edge is characterized by a different safety 
function $S_e$. In our experiments we considered safety functions with a sigmoidal shape (as in Figure \ref{fig:sfunction}). 
The parameters $t_{\mathrm{min}}$ and $t_{\mathrm{max}}$ (as defined in Figure \ref{fig:sfunction}) were randomly selected (in a way, of course, that $t_{\mathrm{min}}<t_{\mathrm{max}}$) and are in 
general different for each edge. This particular choice of the safety function $S_e$ is meant to be just an example, and in practice one can consider any safety function as long as it obeys the constraints outlined in Section \ref{sec:formulation}. In all experiments we assume that the constraint cost uncertainty is upper bounded by $ \overline\epsilon(x,a)=0.5\,d(x,a)$ for all $(x,a)\in\mathcal K^\prime$.

\begin{figure}[htb]
\centering
\includegraphics[height=0.88 \columnwidth]{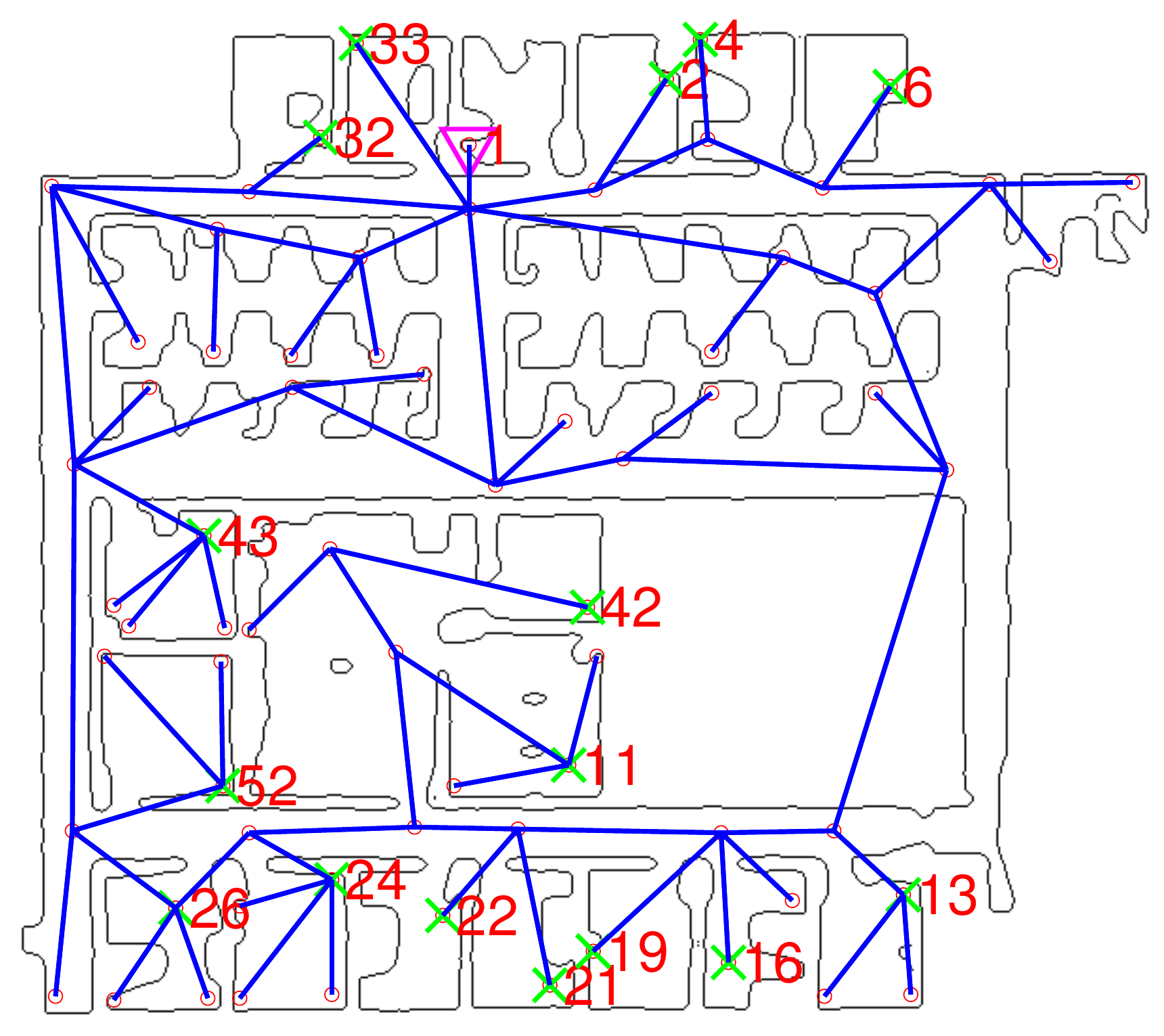}
\caption{The map used to experimentally evaluate the deployment
policies is the same as the one used in \cite{CarpinICRA2013}.  
The deployment vertex is marked with a pink triangle,
whereas goal vertices are indicated by green crosses. Edges between vertices indicate
that a path exists.}
\label{fig:maps}
\end{figure}


\subsection{Single-Robot Deployment Problem}\label{subsec:single}

In this section we numerically investigate the rapid single-robot deployment problem along two dimensions: convergence of the empirical success probability to its theoretical value (equal to $\rho(\mathcal S, a_{\mathcal S})$), and fulfillment of temporal constraint \emph{in expectation}. Collectively, these experiments are aimed at showing the correctness of our approach (i.e., of Algorithm \ref{alg:algorithm_different_s2}) for the single-robot case. For all experiments in this section, we consider an uncertainty budget 
\[
\Gamma=1 \times\sum_{(x,a)\in\mathcal K^\prime} \,\overline \epsilon(x,a),
\]  
and we assume that vertex 13 is the target vertex (a parametric study for different values of $\Gamma$ is presented in Section \ref{subsec:bunc}).

%

To study  the convergence of the empirical success probability with respect to the number of Monte Carlo trials, we  consider a temporal deadline $D=237$. The convergence error is defined as:
\[
\mathrm{error} := \frac{|\mathrm{empirical \,\, \, PF} - \mathrm{theoretical \,\, \, PF}| }{\mathrm{theoretical \,\, \, PF}}.
\]
Alternatively, one can consider the binary random variable $\mathcal M\in\{\text{success},\text{fail}\}$, and then study the Kullback-Leibler (KL) divergence for the probability mass function associated with $\mathcal M$, that is:
\[
D_{\text{KL}}=  \ln\left(\frac{P_{\mathcal M}}{1-\rho(\mathcal S, a_{\mathcal S})}\right) P_{\mathcal M}+ \ln\left(\frac{1-P_{\mathcal M}}{\rho(\mathcal S, a_{\mathcal S})}\right) (1-P_{\mathcal M}),\!
\]
where $P_{\mathcal M}$ is the empirical success probability and $1-\rho(\mathcal S, a_{\mathcal S})$ is the theoretical success probability.
The results are presented in Table \ref{table:convergence}. One can note that (i) the empirical success probability converges to the theoretical success probability, as expected, and (ii) with 100-1,000 Monte Carlo trials convergence is already satisfactory. Accordingly, for the numerical experiments in the remainder of this section and for those in Section \ref{subsec:bunc} we will use a number of Monte Carlo trials equal to $1,000$, while for the numerical experiments in Section \ref{subsec:swarms} we will use a number of Monte Carlo trials equal to $100$.

\begin{table}
\centering
{\small
\begin{tabular}{|c|c|c|c|c|}
        \hline 
 & MC = 100   & MC = 1,000  & MC = 5,000     & MC 10,000     \\
        \hline 
Error &$4.82\%$&$1.37\%$&$0.57\%$&$0.46\%$\\       
$D_{\text{KL}}$ &$0.0146$&$0.0011$&$0.0002$&$0.0001$\\       
 \hline
\end{tabular}
}
\caption{Convergence analysis for empirical success probability (MC stands for number of Monte Carlo trials).}
\label{table:convergence}
\end{table}

We next study how the temporal constraint is fulfilled \emph{in expectation}.
We consider different values for the temporal deadlines, namely $D=\{175,237, 299\}$. 
 Results are presented in Table \ref{table:fulfillment}. The first and second rows of the table report, respectively, the theoretical and the empirical success probabilities, which, as expected, increase as the temporal threshold  is increased. The third row reports the values for the constraint costs, which are always lower (in expectation) than their corresponding  thresholds. The fourth row reports standard deviations. One can note that the standard deviations are rather large: this is due to the fact that, in our formulation, the duration of a deployment task is given by the elapsed time between the instant when the robot starts moving and the instant when it stops moving, respectively. Such duration has a large spread (since failures might induce \emph{early} termination), and the standard deviation is consequently quite large. In the fifth row, we report the expected time to successfully complete the deployment task, conditioned on incurring no failures. For $D=175$, such time is slightly higher than the threshold. This mismatch is explained by the relatively high failure probability shown in the first row of the
 table. On the contrary,  for higher values of $D$ the probability of failure is lower and
 the temporal deadline is met. These findings confirm our discussion in Section \ref{subsec:singlerobot} about the relation between our formulation and a formulation where one constrains the duration of the deployment task \emph{assuming no failures}. Finally, in the last row of the table one can notice  very low values for the standard deviations (again, conditioned on no failures): this is due to two facts (i) assuming no failures, the only source of randomness is the randomization of the control policies, and (ii) according to \cite[Theorem 3.8]{Altman1996}, with one constraint (since there is no uncertainty) an optimal policy only uses \emph{one} randomization. 

\begin{table}
\centering
{\small
\begin{tabular}{|l|c|c|c|}
        \hline 
  & $D=175$ & $D =237 $     & $D=299$    \\
        \hline 
        Theoretical success prob. &$0.5356$ &$0.7321$ &$0.9172$ \\
       Empirical success prob. &$0.5350$&$0.7370$& $0.9140$\\      
Expectation &$115.2550$&$158.5350$&$199.7270$\\ 
Standard dev. &$94.8038$&$72.5002$&$34.0615$\\ 
Expectation (success)&$200$&$199.8225$&$207.9409$\\    
Standard dev. (success) &$0$&$0.5441$&$1.3213$\\   
 \hline
\end{tabular}
}
\caption{Probability of success and fulfillment of temporal constraint for single-robot deployment.}
\label{table:fulfillment}
\end{table}

\subsection{Sensitivity with Respect to Budget Uncertainty}\label{subsec:bunc}
In this section we study the sensitivity of the solution to the single-robot deployment problem with respect to different values of the uncertainty budget. Specifically, we parameterize the uncertainty budget as: 
\begin{equation*}
\begin{split}
\Gamma&=\gamma \times\sum_{(x,a)\in\mathcal K^\prime} \, \overline \epsilon(x,a),
\end{split}
\end{equation*}
where $\gamma$, referred to as ``factor of uncertainty", takes the values shown in the first column of Table \ref{table:factor}. We assume, as before, that vertex 13 is the target vertex and we consider a constraint threshold  $D = 237$. The results are reported in Table \ref{table:factor}. One can observe that, for this example,  the success probability decreases by about 25\% when going from the nominal model $\gamma = 0$ to the worst-case perturbation model ($\gamma=1$). Interestingly, the success probability decreases steeply as soon as a small amount of uncertainty is allowed (say, $\gamma=0.01$, or 1\% uncertainty budget) and then it saturates. In general, this analysis would allow a system designer to assess the ``robustness" of the deployment protocol.

\begin{table}
\centering
{\footnotesize
\begin{tabular}{|c|c|c|}
        \hline 
        Uncertainty budget  & Theoretical succ. prob. & Empirical succ. prob.\\ \hline 
        $\gamma=0$&$0.9853$ &$0.9840$  \\\hline
         $\gamma=0.005$ &$0.8850$&$0.8910$\\      \hline
          $\gamma=0.0075$ &$0.8390$&$0.8350$\\      \hline
           $\gamma=0.01$ &$0.8227$&$0.8220$\\      \hline
           $\gamma=0.0125$ &$0.8166$&$0.8120$\\      \hline
                      $\gamma=0.025$ &$0.7464$&$0.7450$\\     \hline
        $\gamma=0.25$ &$0.7321$&$0.7340$\\      \hline
$\gamma=1$&$0.7321$&$0.7310$\\      
 \hline
\end{tabular}
}
\caption{Sensitivity of success probability with respect to factor of uncertainty $\gamma$. }
\label{table:factor}
\end{table}

\subsection{Deployment of Robotic Swarms}\label{subsec:swarms}
Finally, in this section we study  how the success probability scales with the number of robots (we recall that the multi-robot deployment is considered successful if  each target is reached by at least one robot). We consider an uncertainty budget:
\[
\Gamma=0.25\times\sum_{(x,a)\in\mathcal K^\prime} \, \overline\epsilon(x,a).
\]  

First, we test the strategy whereby targets are randomly assigned to robots according to a uniform probability distribution. We consider as temporal deadlines $D=\{20,51,82,113,144,175, 206, 237,268,299\} $. The results are presented in Figure \ref{fig:multi_easy}. One can notice that, as expected, the success probability increases with the number of robots and time threshold. We believe that this type of plots can be very valuable in practical applications, as one can then decide upfront how to 
pick the size of the team based on a given temporal deadline and a desired probability of success. Figure \ref{fig:multi_easy} indeed describes
the interplay between the size of the team, the temporal deadline, and the probability of success. For example, by using the RCMDP algorithm with uniform random target assignment, for a time threshold equal to $175$, one would require at least $118$ robots in order to achieve a success probability larger than $0.8$. These results do not consider congestion effects: inclusion of upper bounds on the number of robots that can traverse each edge of a map (that is, edge capacities) will be studied in future research.

\begin{figure}[htb]
\centering
\includegraphics[width=0.85\linewidth]{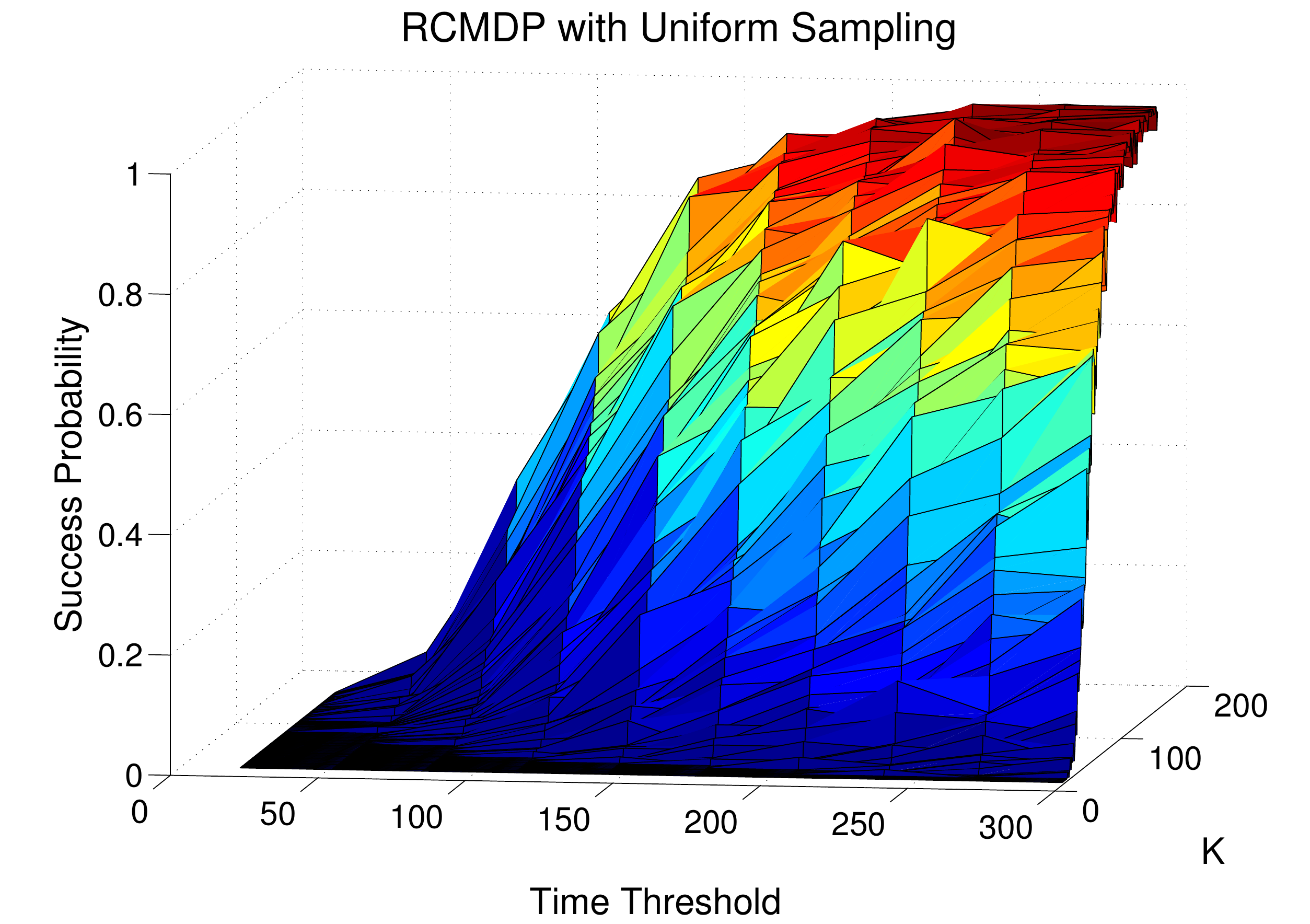}
\caption{Success rate as a function of the number of robots for different temporal deadlines  using random uniform assignment. }
\label{fig:multi_easy}
\end{figure}

We next test the strategy whereby targets are assigned to robots by executing Algorithm \ref{alg:algorithm_different_s} (presented in Section \ref{sec:Efficient_TA}). Results are presented in Figure \ref{fig:multi_smart}. As in the previous case, the success probability increases with the number of robots and time threshold, but it does so much faster. For example,  with a time threshold equal to  $175$, one would require at least  $69$ robots (instead of $118$ as in the previous case) in order to achieve a success probability larger than $0.8$. Visually, this improved performance is evidenced by the steeper shape of the surface when
compared with the analogous one in Figure \ref{fig:multi_easy}.

\begin{figure}[htb]
\centering
\includegraphics[width=0.85\linewidth]{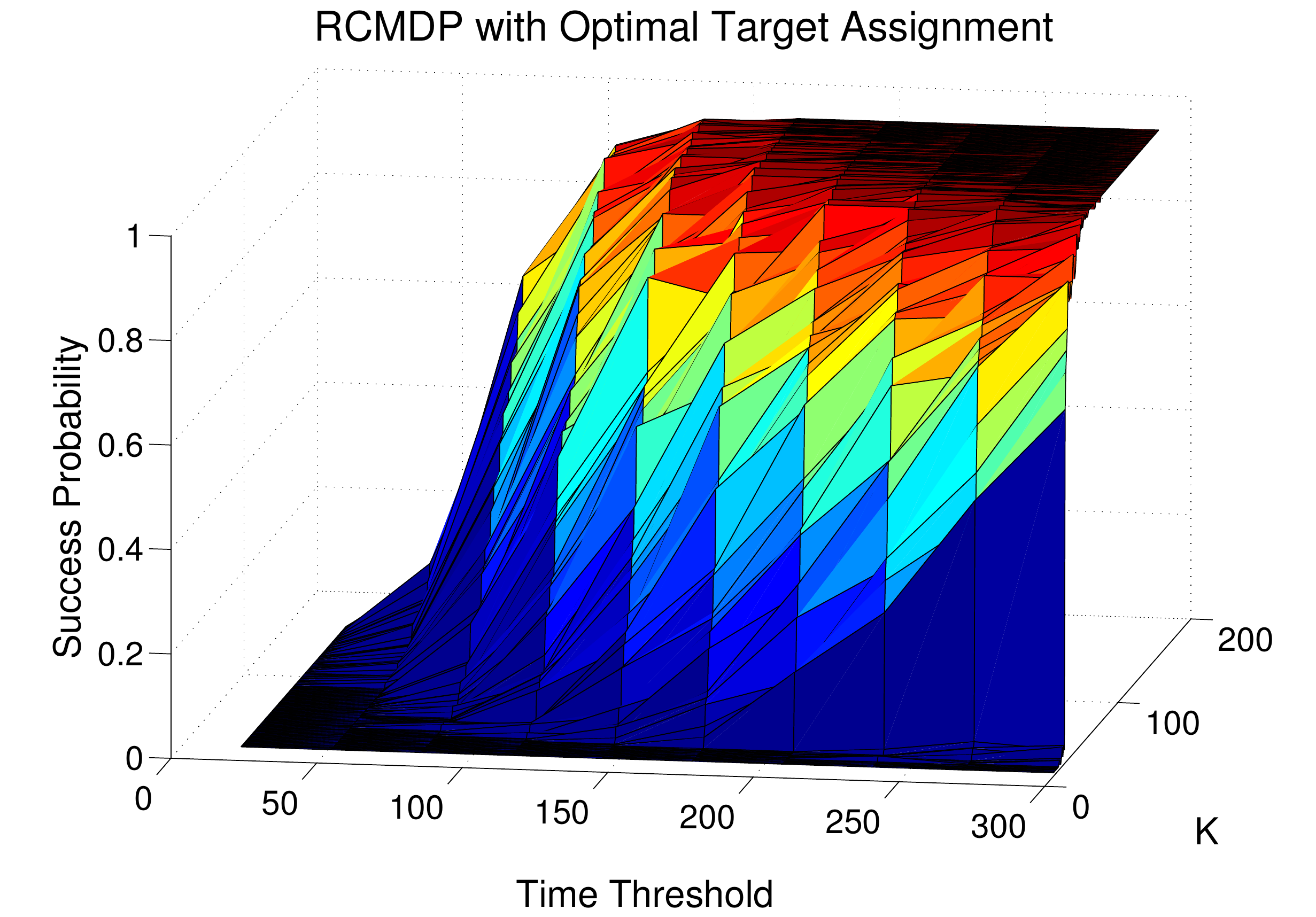}
\caption{Success rate as a function of the number of robots for different temporal deadlines  using optimal target assignment. }
\label{fig:multi_smart}
\end{figure}

 \section{Conclusions}
\label{sec:conclusions}
In this paper we studied the rapid multi-robot deployment problem, where there is an inherent trade-off between mission success and speed of execution. We showed how the rapid deployment problem can be formulated within the theory of constrained Markov Decision Processes, whereby one seeks to compute policies that maximize the probability of successful deployment while ensuring that the expected duration of the deployment task is bounded by a given deadline. To account for uncertainties in the problem parameters, we considered a robust formulation and we proposed efficient solution algorithms, which are novel and of independent interest. Numerical experiments corroborated our findings and showed how
the algorithmic machinery we introduced can be used to address relevant design questions.
For example, it is possible to anticipate the performance of a team of a given size, or 
to decide the size of a team in order to achieve certain performance objectives.

Future research will develop in two directions. From the point of view of deployment, 
methods relying on explicit communication and coordination between the 
agents are of course of interest and will be investigated. From the point of view
of modeling, additional or different types of uncertainties will be considered.
In particular, it is of interest to study the impact of uncertainty in the objective costs
and/or on the transition probabilities. This would for example be of interest when
the safety functions characterizing the edges are only coarsely estimated.

\begin{acknowledgment}
Stefano Carpin is partially supported by ARL under contract MAST-SUPP-13-6-CNC.
Any opinions, findings, and conclusions or recommendations expressed in these materials are those of
he authors and should not be interpreted as representing the official policies, either expressly or
implied, of the funding agencies of the U.S. Government.

The authors thank the reviewers, in particular Reviewer \#2, for numerous thoughtful comments that helped to improve the quality of this paper.
\end{acknowledgment}

%
% ...

\bibliographystyle{asmems4}

\bibliography{asmepaperbib}

\begin{thebibliography}{10}

\bibitem{Bonabeau.ea:99}
Bonabeau, E., Dorigo, M., and Theraulaz, G., 1999.
\newblock {\em {Swarm Intelligence: from Natural to Artificial Systems}},
  Vol.~4.
\newblock Oxford University Press New York.

\bibitem{BulloRN}
Bullo, F., Cort\'es, J., and Mart\'inez, S., 2009.
\newblock {\em {Distributed Control of Robotic Networks}}.
\newblock Princeton.

\bibitem{Pavlic2009}
Pavlic, T., and Passino, K., 2009.
\newblock ``{Foraging Theory for Autonomous Vehicle Speed Choice}''.
\newblock {\em {Engineering Applications of Artificial Intelligence}, {\bf
  22}}(3), pp.~482--489.

\bibitem{cortes2004coverage}
Cortes, J., Martinez, S., Karatas, T., and Bullo, F., 2004.
\newblock ``{Coverage Control for Mobile Sensing Networks}''.
\newblock {\em IEEE Transactions on Robotics and Automation, {\bf 20}}(2),
  pp.~243--255.

\bibitem{Schwager.ea:RSS06}
Schwager, M., McLurkin, J., and Rus, D., 2006.
\newblock ``{Distributed Coverage Control with Sensory Feedback for Networked
  Robots}''.
\newblock In Proceedings of Robotics: Science and Systems.

\bibitem{Morlok2004}
Morlok, R., and Gini, M., 2004.
\newblock ``{Dispersing Robots in an Unknown Environment}''.
\newblock In International Symposium on Distributed Autonomous Robotic Systems
  (DARS).

\bibitem{scout}
Pearce, J., Rybski, P., Stoeter, S., and Papanilolopoulos, N., 2003.
\newblock ``{Dispersion Behaviors for a Team of Multiple Miniature Robots}''.
\newblock In IEEE International Conference on Robotics and Automation,
  pp.~1158--1163.

\bibitem{pavone2011distributed}
Pavone, M., Arsie, A., Frazzoli, E., and Bullo, F., 2011.
\newblock ``{Distributed Algorithms for Environment Partitioning in Mobile
  Robotic Networks}''.
\newblock {\em IEEE Transactions on Automatic Control, {\bf 56}}(8),
  pp.~1834--1848.

\bibitem{CarpinICRA2013}
Carpin, S., Chung, T., and Sadler, B., 2013.
\newblock ``{Theoretical Foundations of High-Speed Robot Team Deployment}''.
\newblock In {IEEE} International Conference on Robotics and Automation,
  pp.~2025--2032.

\bibitem{Kloetzer.Belta:TR07}
Kloetzer, M., and Belta, C., 2007.
\newblock ``{Temporal Logic Planning and Control of Robotic Swarms by
  Hierarchical Abstractions}''.
\newblock {\em IEEE Transactions on Robotics, {\bf 23}}(2), pp.~320--330.

\bibitem{ding2011automatic}
Ding, X., Kloetzer, M., Chen, Y., and Belta, C., 2011.
\newblock ``{Automatic Deployment of Robotic Teams}''.
\newblock {\em IEEE Robotics \& Automation Magazine, {\bf 18}}(3), pp.~75--86.

\bibitem{BatalinTraRO2007}
Batalin, M., and Sukhatme, G., 2007.
\newblock ``{The Design and Analysis of an Efficient Local Algorithm for
  Coverage and Exploration Based on Sensor Network Deployment}''.
\newblock {\em IEEE Transactions on Robotics, {\bf 23}}(4), pp.~661--675.

\bibitem{Fink2013}
Fink, J., Ribeiro, A., and Kumar, V., 2013.
\newblock ``{Robust Control of Mobility and Communications in Autonomous Robot
  Teams}''.
\newblock {\em IEEE Access, {\bf 1}}, pp.~290--309.

\bibitem{matignon2012coordinated}
Matignon, L., Jeanpierre, L., and Mouaddib, A., 2012.
\newblock ``{Coordinated Multi-Robot Exploration Under Communication
  Constraints Using Decentralized Markov Decision Processes}''.
\newblock In AAAI Conference on Artificial Intelligence, pp.~2017--2023.

\bibitem{ding2013strategic}
Ding, X., Pinto, A., and Surana, A., 2013.
\newblock ``{Strategic Planning under Uncertainties via Constrained Markov
  Decision Processes}''.
\newblock In IEEE International Conference on Robotics and Automation, IEEE,
  pp.~4568--4575.

\bibitem{Klavis2011}
Napp, N., and Klavins, E., 2011.
\newblock ``{A Compositional Framework for Programming Stochastically
  Interacting Robots}''.
\newblock {\em International Journal of Robotics Research, {\bf 30}}(6),
  pp.~713--729.

\bibitem{Altman1996}
Altman, E., 1996.
\newblock ``{Constrained Markov decision Processes with Total Cost Criteria:
  Occupation Measures and Primal LP}''.
\newblock {\em Mathematical Methods of Operations Research, {\bf 43}}(1),
  pp.~45--72.

\bibitem{BertsekasDPVol1}
Bertsekas, D., 2005.
\newblock {\em {Dynamic Programming \& Optimal Control}}, Vol.~1 and 2.
\newblock Athena Scientific.

\bibitem{PutermanMDP}
Puterman, M., 1994.
\newblock {\em {Markov Decision Processes -- Discrete Stochastic Dynamic
  Programming}}.
\newblock Wiley-Interscience.

\bibitem{altman1999constrained}
Altman, E., 1999.
\newblock {\em {Constrained Markov Decision Processes}}.
\newblock Stochastic modeling. Chapman \& Hall/CRC.

\bibitem{ben2009robust}
Ben-Tal, A., El~Ghaoui, L., and Nemirovski, A., 2009.
\newblock {\em {Robust Optimization}}.
\newblock Princeton University Press.

\bibitem{DB-MS:03}
Bertsimas, D., and Sim, M., 2003.
\newblock ``{Robust Discrete Optimization and Network Flows}''.
\newblock {\em Mathematical Programming, {\bf 98}}(1-3), pp.~49--71.

\bibitem{rausand2004system}
Rausand, M., and H{\o}yland, A., 2004.
\newblock {\em {System Reliability Theory: Models, Statistical Methods, and
  Applications}}, Vol.~396.
\newblock John Wiley \& Sons.

\bibitem{YN:AN-94}
Nesterov, Y., Nemirovskii, A., and Ye, Y., 1994.
\newblock {\em {Interior-point Polynomial Algorithms in Convex Programming}},
  Vol.~13.
\newblock Society for Industrial and Applied Mathematics.

\bibitem{LuenbergerBook}
Luenberger, D., 2003.
\newblock {\em {Linear and Nonlinear Programming}}.
\newblock Kluwer Academic Press.

\bibitem{Purohit2011a}
Purohit, A., and Zhang, P., 2011.
\newblock ``{Controlled-mobile Sensing Simulator for Indoor Fire Monitoring}''.
\newblock In Wireless Communications and Mobile Computing Conference,
  pp.~1124--1129.

\bibitem{Kumar01092012}
Kumar, V., and Michael, N., 2012.
\newblock ``{Opportunities and Challenges with Autonomous Micro Aerial
  Vehicles}''.
\newblock {\em The International Journal of Robotics Research, {\bf 31}}(11),
  pp.~1279--1291.

\bibitem{AK:DM:12}
Kushleyev, A., Mellinger, D., and Kumar, V., 2012.
\newblock ``{Towards a Swarm of Agile Micro Quadrotors}''.
\newblock In Robotics: Science and Systems.

\bibitem{DM:NM:12}
Mellinger, D., Michael, N., and Kumar, V., 2012.
\newblock ``{Trajectory Generation and Control for Precise Aggressive Maneuvers
  with Quadrotors}''.
\newblock {\em The International Journal of Robotics Research, {\bf 31}}(5),
  pp.~664--674.

\end{thebibliography}

\section*{Appendix}

\begin{proof}[Proof of Lemma \ref{lemma:sup}]
Fix a policy $\pi$ and an initial distribution $\beta$. First, we note that
\begin{equation}\label{eq:sup_in}
 \sup_{\epsilon\in \mathcal U} \, d_{\epsilon}(\pi,\beta) =  \sum_{t=0}^{\infty} \, E_{\pi,\beta} \, [d(X_t,A_t)] +  \sup_{\epsilon\in \mathcal U} \, \sum_{t=0}^{\infty} \, E_{\pi,\beta} \, [\epsilon(X_t, A_t)],
\end{equation}
which follows from the additivity property of the series and expectation operator, and the fact that both series converge given our standing assumption of $\mathbf{X}'$-transient CMDP. Next we define  $\psi^{\pi, \beta}(\epsilon):=\sum_{t=0}^{\infty} \, E_{\pi,\beta} \, [\epsilon(X_t, A_t)]$  and we show that it is a linear functional. Indeed, one has for all $\epsilon_1, \epsilon_2\in \reals^{| \mathcal{K}'|}$:
\begin{equation*}
\begin{split}
\psi^{\pi, \beta}(\epsilon_1 + \epsilon_2) &= \sum_{t=0}^{\infty} \, E_{\pi,\beta} \, [\epsilon_1(X_t, A_t)+\epsilon_2(X_t, A_t)]\\
&=\sum_{t=0}^{\infty} \, E_{\pi,\beta} \, [\epsilon_1(X_t, A_t)]+\sum_{t=0}^{\infty} \, E_{\pi,\beta} \,  [\epsilon_2(X_t, A_t)]\\
&= \psi^{\pi, \beta}(\epsilon_1)+\psi^{\pi, \beta}(\epsilon_2),
\end{split}
\end{equation*}
where the second equality again follows from the additivity property of the  series and expectation operators and the fact that both series converge. Analogously, one has for all $\epsilon\in \reals^{| \mathcal{K}'|}$ and  $\alpha \in \reals$
\begin{equation*}
\begin{split}
\psi^{\pi, \beta}(\alpha \, \epsilon) &= \sum_{t=0}^{\infty} \, E_{\pi,\beta} \, [\alpha \, \epsilon(X_t, A_t)]\\
&=\alpha \, \sum_{t=0}^{\infty} \, E_{\pi,\beta} \, [  \epsilon(X_t, A_t)] = \alpha \, \psi^{\pi, \beta}(\epsilon). 
\end{split}
\end{equation*}
Hence, the functional $\psi^{\pi, \beta}$ is linear. Since $\mathcal U$ is a finite dimensional set and $\psi^{\pi, \beta}$ is a linear functional, it follows that $\psi^{\pi, \beta}$ is continuous. Note that $\mathcal U$ is closed and bounded, hence by the extreme value theorem the functional $\psi^{\pi, \beta}$ achieve its maximum. Combining this fact with equation \eqref{eq:sup_in}, one obtains the claim.
\end{proof}

\begin{proof}[Proof of theorem \ref{thm:OM}]
As in the proof of Lemma \ref{lemma:sup}, let $\psi^{\pi, \beta}(\epsilon):=\sum_{t=0}^{\infty} \, E_{\pi,\beta} \, [\epsilon(X_t, A_t)]$. As shown in Lemma \ref{lemma:sup},  $\psi^{\pi, \beta}(\epsilon)$ is a linear functional with respect to $\epsilon$. The constraint $\max_{\epsilon\in \mathcal U} \, d_{\epsilon}(\pi,\beta) \leq D$ is clearly equivalent to the constraint
\[
  \sum_{t=0}^{\infty} \, E_{\pi,\beta} \, [d(X_t,A_t)]    + \max_{\epsilon\in \mathcal U} \,  \psi^{\pi, \beta}(\epsilon) \leq D.
\]
Note that $\mathcal U$ is the intersection of a hyper-rectangle, i.e., $\{\epsilon: 0\leq \epsilon(x,a)\leq \overline\epsilon(x,a),\forall (x,a)\in\mathcal K^\prime\}$, and a simplex, i.e., $\{\epsilon:\sum_{(x,a)\in \mathcal K'} \, \epsilon(x,a)\leq \Gamma\}$. Thus, $\mathcal U$ has a finite number of vertices.
Let $\mathcal V$ be the set of vertices of the polytopic set $\mathcal U$ and let $|\mathcal V|$ be its cardinality. Since the maximum for a linear optimization problem over a bounded polytopic  set is achieved at one of the vertices, then the constraint
\[
  \sum_{t=0}^{\infty} \, E_{\pi,\beta} \, [d(X_t,A_t)]    + \max_{\epsilon\in \mathcal U} \,  \psi^{\pi, \beta}(\epsilon) \leq D,
\]
is equivalent to 
\[
 \sum_{t=0}^{\infty} \, E_{\pi,\beta} \, [d(X_t,A_t)]    + \psi^{\pi, \beta}(\epsilon_j) \leq D \quad  \text{ for all } \epsilon_j \in \mathcal V.
\]

Therefore, problem RCMDP is equivalent to the problem:
\begin{align*}
&\min_{\pi\in\Pi_{\text{M}}} c(\pi,\beta) \label{COP} \\
&\textrm{s.t.} \sum_{t=0}^{\infty} \, E_{\pi,\beta} \, [d(X_t,A_t) +\epsilon_j(X_t, A_t)] \leq D \quad  \text{ for all } \epsilon_j \in \mathcal V,
\end{align*}
where the number of constraints is $|\mathcal V|$. According to Theorem 8.1 in \cite{Altman1996}, this problem is equivalent to the problem
\begin{alignat*}{2}
\min_{\rho} & & \quad&  \sum_{(x,a) \in \mathcal{K}'}\rho(x,a)c(x, a)\\  
\text{s.t.} & & \quad&\sum_{y\in \mathbf{X'}}\sum_{a\in A(y)}\rho(y,a)\left[\delta_x(y)- \mathcal{P}_{yx}^a\right]= \beta(x), \forall x\in \mathbf{X'}\\
& &\quad &  \!\!\!\!\!\!\!\!\sum_{(x,a) \in \mathcal{K}'}\rho(x,a)(d(x, a)+\epsilon_j(x,a)) \leq D    \text{ for all } \epsilon_j \in \mathcal V\\  
& &\quad &\rho(x,a)\geq 0,\,\, \forall (x,a)\in \mathcal{K}'.
\end{alignat*}
Equivalency is in the sense that the solution to the above linear optimization problem induces an optimal policy, see equation \eqref{eq:induced} in section \ref{subsection:CMDP}.

Finally, the above problem can be rewritten as 
\begin{alignat*}{2}
\min_{\rho}  & & \quad&  \sum_{(x,a) \in \mathcal{K}'}\rho(x,a)c(x, a)\\  
\text{s.t.} & & \quad&\sum_{y\in \mathbf{X'}}\sum_{a\in A(y)}\rho(y,a)\left[\delta_x(y)- \mathcal{P}_{yx}^a\right]= \beta(x), \forall x\in \mathbf{X'}\\
& &\quad &\max_{\epsilon\in \mathcal U}\sum_{(x,a) \in \mathcal{K}'}\rho(x,a)(d(x, a)+\epsilon(x,a)) \leq D   \\  
& &\quad &\rho(x,a)\geq 0,\,\, \forall (x,a)\in  \mathcal{K}'.
\end{alignat*}

This concludes the proof.
\end{proof}

\end{document}